\documentclass[twocolumn]{article}

\usepackage[round]{natbib}

\usepackage{amsmath,amsfonts,bm}

\def\eqref#1{equation~\ref{#1}}

\def\1{\bm{1}}

\DeclareMathAlphabet{\mathsfit}{\encodingdefault}{\sfdefault}{m}{sl}
\SetMathAlphabet{\mathsfit}{bold}{\encodingdefault}{\sfdefault}{bx}{n}

\def\sB{{\mathbb{B}}}

\newcommand{\R}{\mathbb{R}}

\newcommand{\lr}{\alpha}

\usepackage{hyperref}
\usepackage[noabbrev,capitalize]{cleveref}
\usepackage{amssymb}
\usepackage{url}
\usepackage{graphicx}
\usepackage{bbm}
\usepackage{makecell}
\usepackage{booktabs}
\usepackage{xcolor}
\usepackage{subcaption}

\def\imagenet{ImageNet}

\newenvironment{proof}{\paragraph{Proof:}}{$\blacksquare$} 
\def\batch{\sB}
\def\lr{\lambda}
\def\weights{\theta}
\def\sampledist{{\mathcal D}}
\def\sampledfrom{\sim}
\def\normalstd{\mathcal{N}(0, \mathbbm{1})}

\newtheorem{lemma}{Lemma}
\newcommand{\normal}[1]{\mathcal{N}(0, \mathbbm{#1})}
\newcommand{\norm}[1]{\|#1\|}
\newcommand{\norms}[1]{\|#1\|^2}

\definecolor{NavyBlue}{RGB}{0, 110, 184}
\definecolor{VioletBlue}{RGB}{120, 85, 170}
\definecolor{BrickRed}{RGB}{182, 50, 28}
\definecolor{Blue}{RGB}{0, 0, 184}

\newcommand{\pd}[2]{\frac{\partial #1}{\partial #2}}

\newcommand{\internal}[1]{}

\newcommand{\Ex}[1]{E\left[#1\right]}
\newcommand{\Exover}[2]{E_{#1}\left[#2\right]}
\newcommand{\Exoversamples}[1]{\Exover{x \sampledfrom \sampledist}{#1}}

\newcommand{\mat}[1]{\mathrm{#1}}
\def\covariance{\mat{C}}
\def\normaln{\mathcal{N}}
\def\lmat{\mat{A}}
\def\coord{\mat{Q}}
\def\id{\mat{I}}

\title{Training trajectories, mini-batch losses and the curious role of the learning rate.}

\author{Mark Sandler \and Andrey Zhmoginov \and Max Vladymyrov \and Nolan Miller 
\\ 
Google Research \\
{\small \texttt{\{sandler,azhmogin,mxv,namiller\}@google.com}}}
\date{}
\begin{document}

\maketitle
\begin{abstract}
Stochastic gradient descent plays a fundamental role in nearly all applications of deep learning. 
However its ability to converge to a global minimum remains shrouded in mystery. In this paper we propose to study the behavior of the loss function on fixed mini-batches along SGD trajectories. We show that the loss function on a fixed batch appears to be remarkably convex-like. In particular for ResNet the loss for any fixed mini-batch can be accurately modeled by a quadratic function and a very low loss value can be reached in just one step of gradient descent with sufficiently large learning rate. We propose a simple model that allows to analyze the relationship between the gradients of stochastic mini-batches and the full batch. Our analysis allows us to discover the equivalency between iterate aggregates and specific learning rate schedules. In particular, for Exponential Moving Average (EMA) and Stochastic Weight Averaging we show that our proposed model matches the observed training trajectories on ImageNet. Our theoretical model predicts that an even simpler averaging technique,  averaging just two points a many steps apart, significantly improves accuracy compared to the baseline. We validated our findings on ImageNet and other datasets using ResNet architecture. 
\end{abstract}

\vspace{-.3cm}
\section{Introduction}
Stochastic Gradient Descent~(SGD; \citealp{robbins-monro}) played a fundamental role in the rise and success of deep learning. Similarly, the learning rate, the multiplier controlling the magnitude of the weight update, is perhaps the most crucial hyper-parameter that controls the training trajectory \citep{goodfellow2016deep,Smith2015-jg,Bengio2012-oh,Smith2015-jg}. Despite the advances in stochastic optimization methods~\citep{Kingma2014-vr} designed to reduce the need for tuning, and in particular dependence on learning rate, novel learning rate schedules  such as cosine \citep{Loshchilov2016-ig}  and cyclical learning rate schedule \citep{Smith2015-jg} continue to be the subject of active research and are far from being rigorously understood. On the other hand, the impact of stochasticity of the optimization process caused by changing input mini-batches also known as  "batch noise", \citep{Ziyin2021-mj} has been relatively poorly understood. In

\begin{figure}[t]
    \centering
    \includegraphics[width=0.45\textwidth]{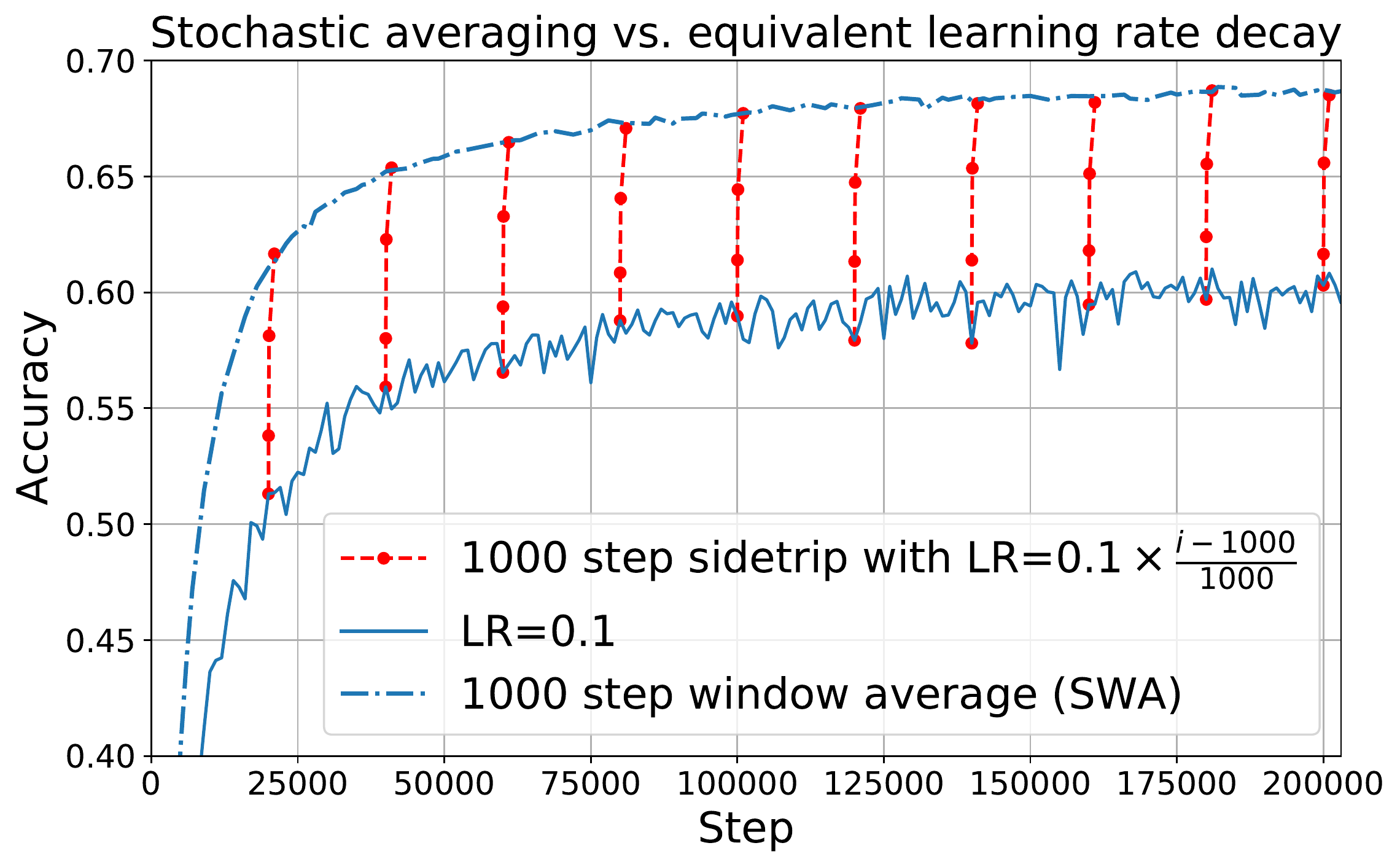}
    \caption{Weight averaging v.s. equivalent learning rate schedule. The dotted vertical red lines show the set of  independent trajectories ``sidetrips'' with appropriate learning rate schedule that start at corresponding point in the main trajectory. }
    \label{fig:momentary-vs-others-top}
 \end{figure}
 
\vspace{-.2cm}
\paragraph{Contributions} In this paper we show that the loss as a function of learning rate  of the training batch is dramatically different than that
on the held-out batch, when measured on SGD trajectory.  This connection to learning rate leads to a simple theoretical model describing observed results. We believe this connection of the learning rate to the loss as a function of held out and training match has been largely overlooked in the literature. 

One practical result that we demonstrate that various popular averaging techniques have equivalent learning rate schedules. To the best of our knowledge such connection has not been observed in the literature. We demonstrate the equivalency in two different scenarios: a simple scenario where we assume that the gradient for the same input changes little between steps, and the other where we show that this equivalency holds in the limit case of an arbitrary long trajectory. We demonstrate (e.g. \cref{fig:momentary-vs-others-top}) that these results hold empirically for a general training of \imagenet~\citep{imagenet} using ResNet architecture~\citep{resnets}.

We propose a simple model that captures these behaviors. The model we consider is similar to that of~\cite{Mandt2017-zy}, but our analysis is noticeably simpler.  We further show that our simple model is suitable for studying training trajectories with a spectrum of time scales ranging from fast to slow and corresponding to long trenches in the loss landscape. We show how model weights quickly converge towards a quasi-stationary distribution in some degrees of freedom, while slowly drifting towards the global minimum in others and derive the expression for the auto-correlation of the averaged trajectory as a function of the averaging kernel (\cref{app:weight-averaging}).

Using this model we develop a simple geometric explanation of why averaging improves the accuracy in the first place. We show that even on large datasets such as \imagenet\ and for deep architectures using the fixed learning rate the trajectory traverses an ellipsoid centered around the target minimum and the averaging moves the trajectory to the inside of that ellipsoid. Further, we show that for two independently trained trajectories that share a common burn-in period the \emph{distance} in parameter space along one trajectory to the final solution of another monotonically decreases as training progresses and asymptotically stabilizes at a fixed value. This suggests that most of the probability mass of the training outcomes is  concentrated on a sphere-like volume. 

Our final insight is the dramatic difference in the behavior of the loss function for the batch that it is being optimized v.s.\ a new (training) batch. We are not aware of any study in the literature that explores this difference. While this is reminiscent of the generalization gap, there is a crucial distinction: we analyze the loss on a different \emph{training} batch, and not on a validation batch. To some degree our work also illuminates the approach of Amortized Proximal Optimization~\citep{Bae2022-tv} that used the loss on the unseen batch referred as \emph{Functional Space Discrepancy}, as a negative regularization force for learning rate selection. Curiously their motivation was to reduce the change of the loss on the unseen batch as a way to minimize the predictions on unseen data. However, as we demonstrate here, such type of adaptive schedule tries to apply asymptotically {\em largest} update, that prevents loss on unseen batches to \emph{increase} and thus the training to diverge.  

Interestingly, such difference in gradients and behavior of two training batches, also indicates that a simple intuition \citep{Ziyin2021-mj, Mandt2017-zy,langevin-lb} that stochastic gradient descent is an approximation to batch gradient could be misleading and thus should be used judiciously, despite it being unbiased estimate of the full gradient. 



\paragraph{Paper organization}
The paper is organized as follows. In the next section we describe some of the \imagenet{} experiments that led us to our theoretical model described in section \cref{sec:model}. In ~\cref{sec:experiments} we show some additional experiments on both real data and synthetic model that we introduced in \cref{sec:model}. Finally in~\cref{sec:conclusions} we outline some of the open questions and future directions.

\section{Related work}
Averaging intermediate SGD iterates in the weight space can be traced back to~\citet{ruppert-averaging}.  In more recent years
it become a common tool for speeding up convergence \citep{kaddour22} and improving accuracy \citep{Izmailov2018-wv}.
It was first directly analyzed in~\citet{Polyak1992-lk} and has been studied in the literature since. For example in~\citet{Mandt2017-zy} the connection between SGD and Bayesian inference was established. In particular they demonstrated that SGD training that initialized within common zone of attraction can be seen as Markov chain with Gaussian stationary distribution for constant learning rate, which is related to our observation about trajectory stabilizing at a fixed distance from the global minimum. However in present work we mainly analyze the behavior of a {\emph single trajectory}. This has immediate advantage that it allows us to reason about {\em finite} trajectory lengths, as well as the connection between learning rate schedules and averaging.  

Another related line of inquiry is \emph{Stochastic Weight Averaging} (SWA) by~\citet{Izmailov2018-wv}, where a trajectory with cyclic learning rates and averaging every $c$ steps is usedas a way to improve accuracy. This in theory is slightly different from Ruppert-Polyak averaging, since not every point on the trajectory is averaged, however, many of the experiments presented in that paper use $c=1$ and we suspect this difference is insignificant.  In that work it was also observed that averaging between trajectory allows to move inside the surface, but no further discussion  has followed.  Additionally, the connection between SWA and generalization was further explored in~\citet{He2019-vx}, where it was showed that averaging introduces bias towards flatter part of the basin and this results in better generalization.  In contrast here our goal is to illuminate the mechanisms by which SWA changes the behavior of empirical loss in the context of {\em stochastic} (rather than full) gradient descent.  Connecting current work and the generalization performance is an interesting open direction. The quadratic theoretical model we consider is reminiscent of that
in~\citet{Schaul_undated-uj}, but our analysis is both simpler than former and more general
than latter. We note that analysis only applies to {\em stochastic} gradient descent. In case of full gradient descent there have been several recent works showing that quadratic approximation model might be toosimplistic~\citep{Ma2022-ft, Damian2022-jx, Cohen21}.


Recent work on basins and local minima connectivity~\citep{merging-models,random-initializations, linear-mode-connectivity} have demonstrated that there is a lot of apparent structure in the loss space, in 
particular it was shown that different initialization lead to different equivalence classes, which can then be mapped between each other using simple permutation of the neurons. That work is complementary to ours in that they show that there are simple ways of moving basins, while we concentrate on properties of a single basin. 


\section{Empirical observations}
\label{sec:empirical}
\begin{figure*}[t]
      \centering
      \includegraphics[width=0.85\textwidth]{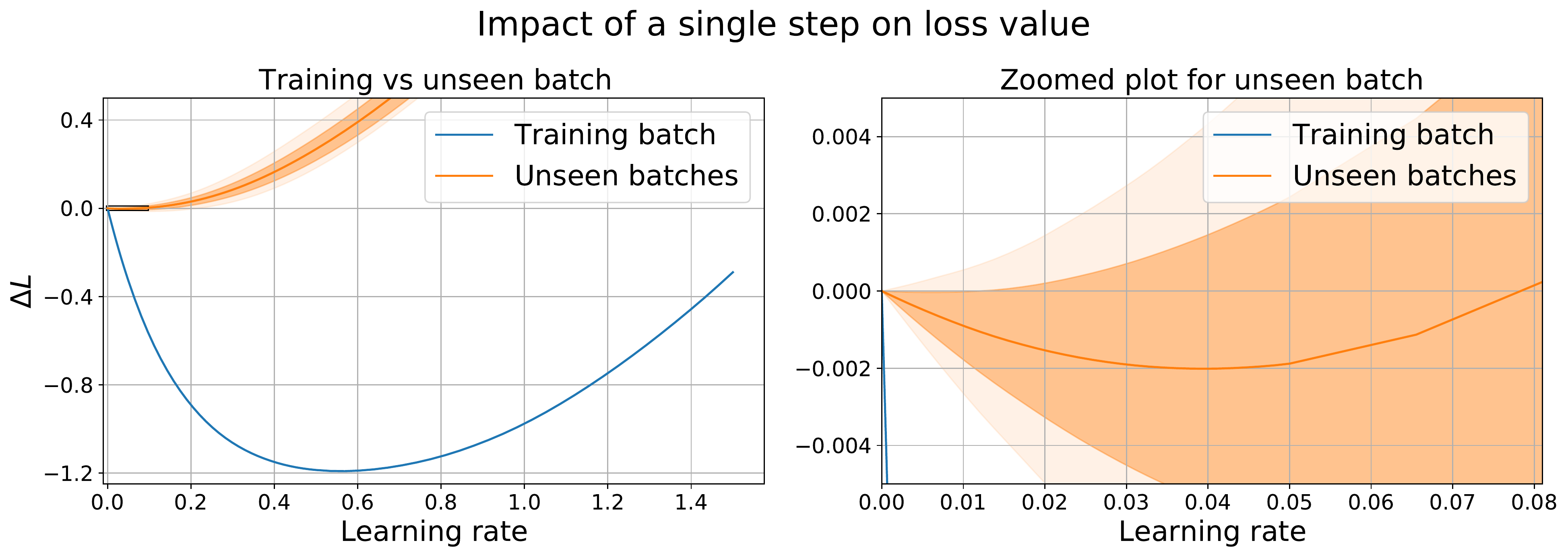}
      \caption{Loss as a function of a step size in a single step gradient descent in the middle of a trajectory. The step shown here is step $75\,000$ out of $100\,000$ run. The behavior is typical for other steps as well. The small rectangle on the left graph
      approximately shows the location of the zoomed-in right graph. The unseen batches is computed over $10$ batches, 
      with dark shade showing the standard deviation from the mean, while light shade shows the max/min values observed.}
      \label{fig:single-step}
\end{figure*}
\begin{figure*}[t]
      \centering
      \includegraphics[width=0.98\textwidth]{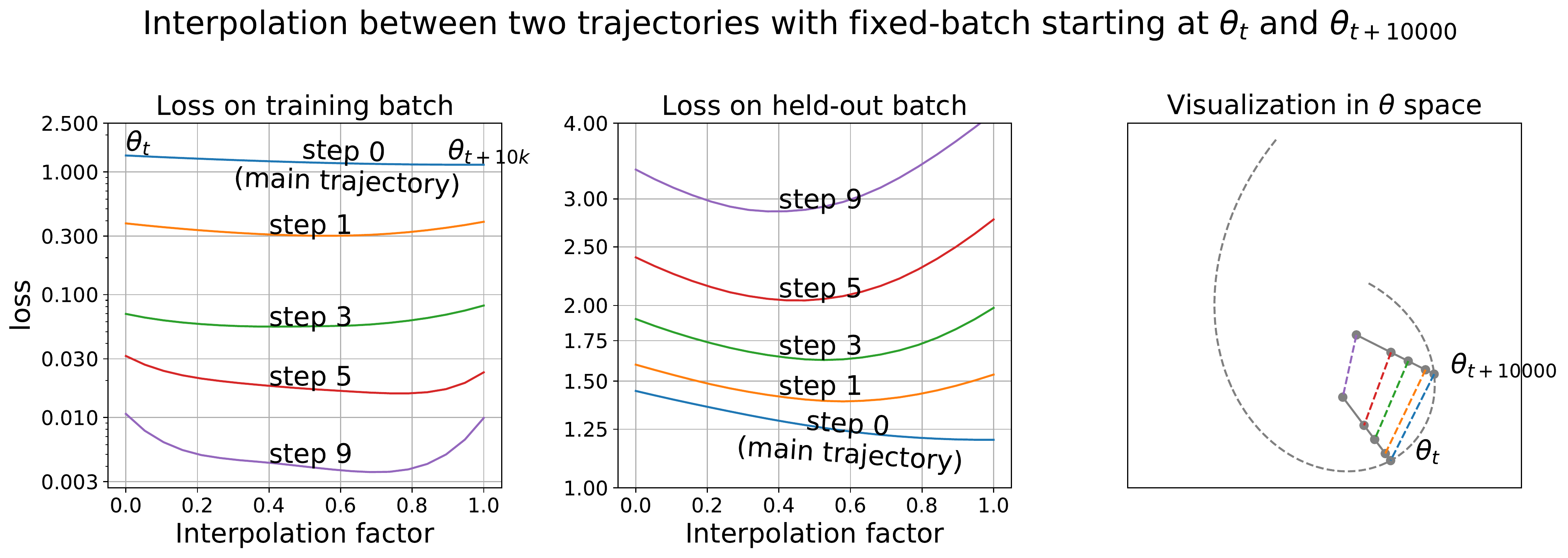}
      \caption{Loss for a fixed batch on the interpolation between two points of the original training trajectory. We use two points along the \imagenet\ training trajectory with $t=75\,000$, and $t=85\,000$.  
      From each point we perform 9 steps of gradient descent on a fixed batch and measure the loss on interpolation between corresponding pairs of points of each trajectory. The left graph shows the loss on the training batch, and note how within just 3 steps it reaches nearly 0. The middle graph shows the loss on held out. At step 0 the training and held out batch loss profiles are very similar. Rightmost graph shows the visualization of the process. }
      \label{fig:interpolation-of-fixed-batch-sidetrip}
\end{figure*}
In this section our goal is to explore some of the properties of SGD that will lead us to the theoretical model of \cref{sec:model}.
\paragraph{Experimental setup} For all our experiments with non-synthettic data 
we use ResNet34~\citep{resnets} architecture and \imagenet~\citep{imagenet}. 
We use standard SGD with momentum $0.9$, that produces good accuracy on \imagenet. 
Most of our experiments on real data involve starting with partially trained architectures, and running a separate \emph{side-trip} with different learning rate characteristics and/or
with a fixed \emph{training} mini-batch to demonstrate certain behaviors. We also will use the notion of \emph{held-out} mini-batch which is simply a \emph{fixed} batch different from the training mini-batch. 
Throughout this section we will be referring to the training trajectory as the \emph{main} trajectory to differentiate from side-trips. 

\vspace{-.2cm}
\paragraph{Single step experiments}
We begin with a simple question: how does a loss landscape look for a single batch
applied with a varying learning rate? Specifically, does the loss behave on
that batch v.s.\ a separate held-out batch? In~\cref{fig:single-step} we
plot the loss v.s.\ the learning rate when applied to a single step
using either the batch for which the direction was computed or a separate held-out
batch. Here we used step 75K, out of 100K,  and batch size 512, 
however very similar results hold elsewhere in the trajectory and for different batch sizes. See~\cref{fig:single-step-batches} for details. 

As can be seen from~\cref{fig:single-step} the loss for both curves is smooth and can be 
well approximated by a low-degree polynomial. More importantly. the loss on the
training batch reaches loss value below the average loss of a fully trained model
in a single step. Typical loss value for fully trained \imagenet{} is close to
$1$, while a single step starting from the loss value $1.8$ results in a 
final loss of $0.5$. The loss on the held-out batch behaves very differently: 
only sufficiently small learning rate leads to any improvement to the loss on the full distribution. 
Therefore any analytical model  purporting to describe SGD trajectories  should have this property.
Remarkably this suggests that full gradient trajectory is not a good approximation of the SGD trajectory,
if we want to understan the dynamics of learning rate selection. 

\vspace{-.2cm}
\paragraph{Multi-step side-trip} As we just observed, a single step along the training batch direction can bring the loss down to values considerably lower than the best average loss of a fully trained model. Naturally one can ask -- does the loss basin for this specific batch remains the same regardless of the starting point? In other words, if we apply gradient descent with a fixed batch   at one point of the main trajectory, would it be the same minima basin as if started from different point of the main trajectory? Note: this question is different from the notion of the global basin connectivity explored in~\citet{linear-mode-connectivity}. Instead, we look at the behavior of the loss for a fixed mini-batch as we branch off different parts of the training trajectory.  Following the analysis of~\citet{linear-mode-connectivity} we use interpolation to check that we are in
the same basin. In~\cref{fig:interpolation-of-fixed-batch-sidetrip}  we show that the side-trip using the same fixed batch along the main training trajectory leads to the same basin. Remarkably, the held-out batch loss while increasing dramatically also stays in the same basin as the main trajectory. 
\vspace{-.2cm}
\section{Analytical model}

\begin{figure*}[t]
      \centering
              \begin{subfigure}{0.4\textwidth}
    \includegraphics[width=0.98\textwidth]{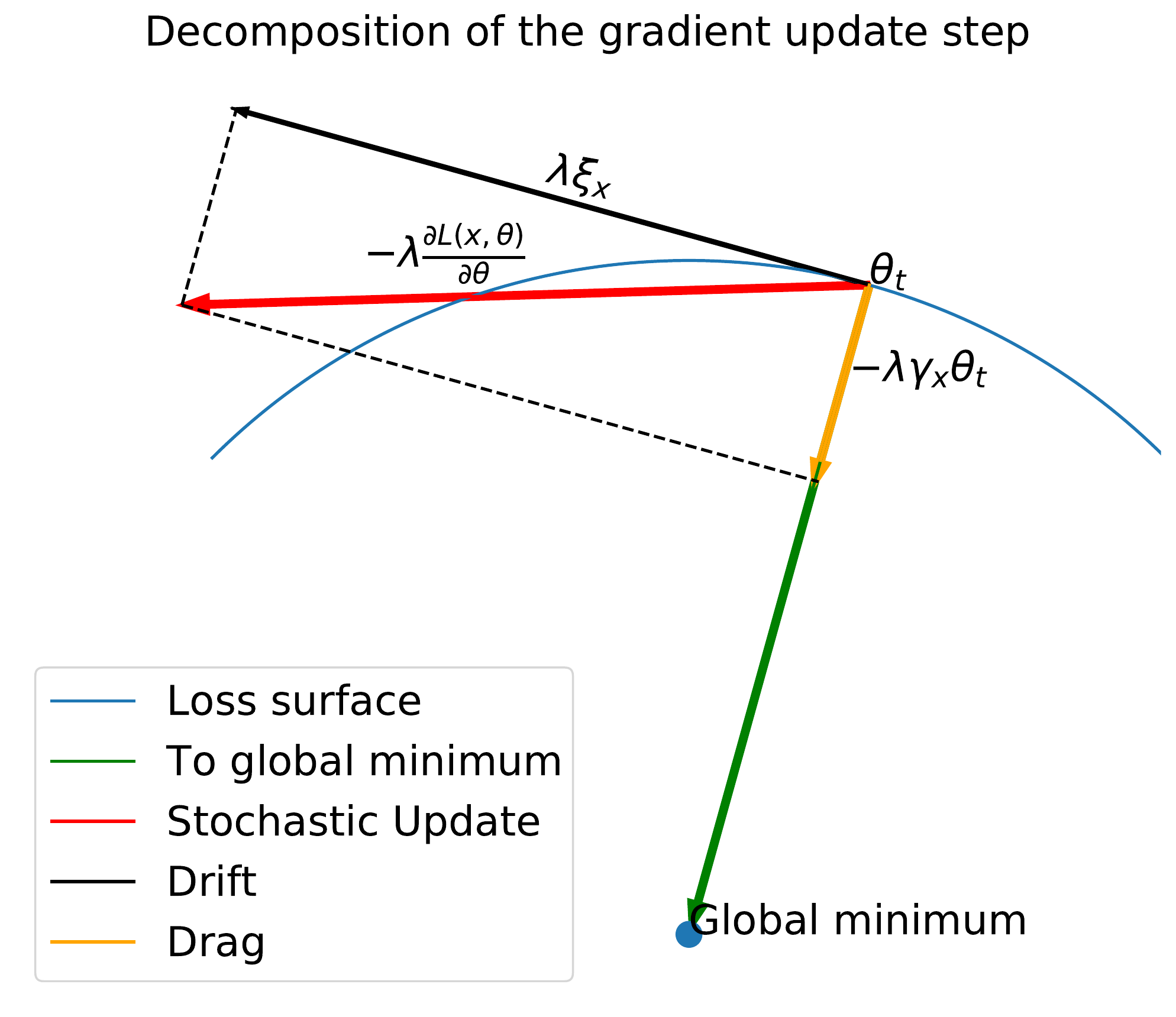}
    \caption{}
          \label{fig:drag-drift-interp-2d}
    \end{subfigure}
    \begin{subfigure}{0.35\textwidth}
    \includegraphics[width=0.98\textwidth]{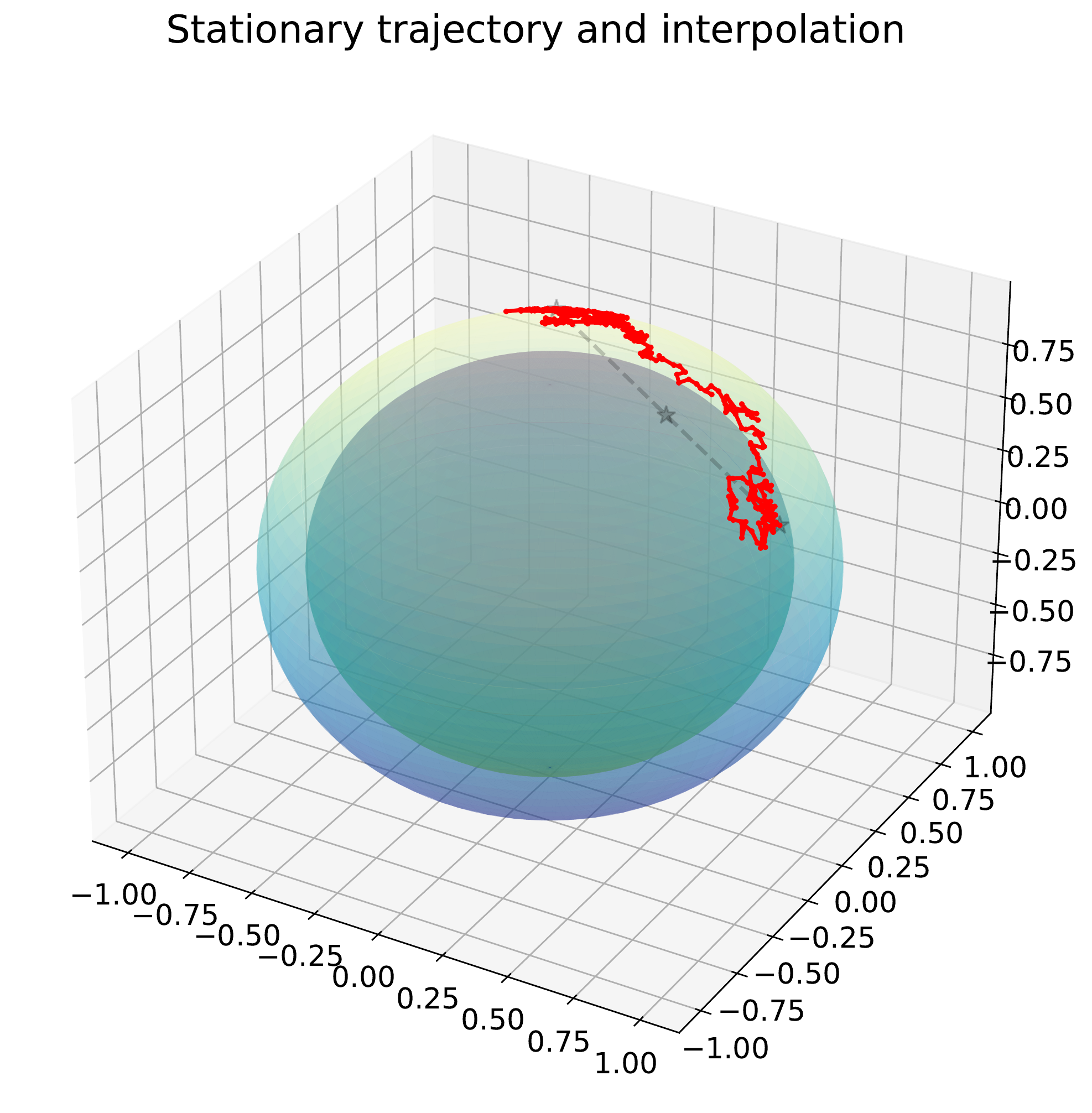}
    \caption{}
    \label{fig:drag-drift-interp-3d}
    \end{subfigure}
      \caption{Stochastic gradient descent. \Cref{fig:drag-drift-interp-2d} shows how stochastic gradient direction decomposes into a drift and a drag component. The magnitude of the drag component changes as $\sqrt{\lambda}$, while
      the magnitude of the drift component is proportional to $\lambda$. Further, the drift component has a quadratic bias to increase loss, thus enabling us to pick learning rate that allows to reduce the loss.  \Cref{fig:drag-drift-interp-3d}  illustrates how stationary trajectory traverses the sphere of a fixed loss, while taking a midpoint enables to get \emph{inside} the sphere. Best viewed in color.}
\end{figure*}
\label{sec:model}
In this section we propose a model describing the behavior of the high dimensional weight vector $\theta$, as it moves along the SGD trajectory, in such a way that it can 
explaining the phenomena that we described in \cref{sec:empirical}: the loss for the individual batches behaves like a low-degree polynomial and reaches remarkably low value in a single step. Thus we can use quadratic function that models the loss function for individual samples: 
$$
 L(\weights, x) = \norms{A_x \weights + c_x} / 2,
$$
where $A_x$ is a matrix and $c_x$ is a vector describing the loss function for a random sample $x \sampledfrom D$.  This model is similar to the one proposed in \citet{Schaul_undated-uj}, however it analyzed only the special case of a constant and diagonal $A_x$, while we consider a general scenario.  Following the mini-batch gradient where individual batches are sampled from some distribution we then will be moving in a  direction whose expectation is $\Exoversamples{\frac{\partial{L(\weights, x)}}{\partial\weights}}=\frac{\partial{\Exoversamples{L(\weights, x)}}}{\partial\weights}.$
Without loss of generality we will assume that $\frac{\partial{\Exoversamples{L(0, x)}}}{\partial\weights} = 0$.


The global loss over the sample distribution $\sampledist$ is 
\begin{align}
    L(\weights) & = \Exoversamples{ L(\weights, x)} \nonumber\\
    & = \theta^T \Exoversamples{ A_x^T A_x} \theta + \Exoversamples{c_x^TA_x}\theta + \nonumber \\
    &\phantom{=}+ \Exoversamples{c_x^T c_x}  \nonumber \\
    & =  \theta^T \Exoversamples{ A_x^T A_x} \theta + \Exoversamples{c_x^T c_x}
\end{align}
where we take expectation over all possible samples. The last transition holds since we assumed that the loss function achieves minimum at $0$, and thus $A_x$ and $c_x$ satisfy $\Ex{A^T_x c_x} = 0$. 

The stochastic gradient and the full gradient can then be written respectively as: 
\begin{equation}
\pd{L(\weights,x)}{\weights} = { A^T_x (A_x \weights + c_x)}.
\end{equation}
\begin{align}
\pd{L(\weights)}{\weights} & = \pd{\Exover{x \sampledfrom \sampledist}{\norms{A_x \weights + c_x} / 2}}{\weights} \\
&=
\Exover{x \sampledfrom \sampledist}{ A^T_x (A_x \weights + c_x)}.   
\end{align}

Now we estimate how close we can get to $0$ for a fixed learning rate. Note, even though the individual steps gradients can be assumed
to be unbiased estimators of global loss gradient, there is generally no guarantee that trajectory will converge to a minimum. Instead
we can show that the trajectory will be {\em traversing} an ellipsoid around the minimum. 

\begin{lemma}
If $\Exoversamples{A^T_xA_x} = I$ the trajectory will stabilize at an ellipsoid of a fixed size proportional to the square root of the learning rate $\sqrt{\lambda}$. 
\end{lemma}
\begin{proof}

Consider a fixed sample $x$ at step $t$. Let $b_x := c_x^T A$. 
\begin{equation}
\label{eq:weight-update}
\weights_{t+1} = \weights_t  - \lr A^T_x A_x \weights_{t} - \lr b_x.
\end{equation}
We would like to describe the behavior of $\norm{\weights_{t+1}} - \norm{\weights_t}$  as a function of $\lr$.
First, we observe that  $\lr \weights^T_tA^T_x A_x \weights_t \ge 0$, with equality achieved only when $A_x\weights_t = 0$. Thus, the second term in \cref{eq:weight-update} can be decomposed into a projection onto $\weights_t$ and its orthogonal component using the inner product with $\weights_t$:
$$
 \lr A^T_x A_x \weights_{t} =  \lr [\gamma_x  \theta_t + \theta^\perp_{t, A_x}]
$$
where $\gamma_x := \frac{\weights^T_t A^T_x A_x \weights_{t}}{\norm{\weights_t}^2}$, and $\weights^\perp_{t, A_x}$ is a vector orthogonal to $\weights_t$. Therefore the update \cref{eq:weight-update} can be rewritten as:
\begin{equation}
\label{eq:updated-weight-update}
\weights_{t+1} = (1 - \lr \gamma_x) \weights_t  - \lr [\theta^\perp_{t, A_x}+ b_x] =  (1 - \lr \gamma_x) \weights_t  - \lr \xi_x
\end{equation}
where 
\begin{equation}\xi_x := \theta^\perp_{t, A_x}+ b_x.
\label{eq:xi-def}
\end{equation} 

The update to $\theta_t$ consists of two terms: (a) the component $-\lr \gamma_x \weights_t,$ whose
magnitude depends on $A_x$ that reduces norm of $\weights$ and (b) an orthogonal drifting term $\xi_x$. The term $\xi_x$ also consists of two components: $\theta^{\perp}_{T, A_x}$ which is orthogonal to $\weights_t$ by construction, 
and $b_x = A^T_x c_x$ which we can assume to be orthogonal to $\weights_t$ since it is independent of the orientation of $\weights_t$. 
Thus, $\xi_x$ is orthogonal to $\weights_t$ and is determined by $A_x, c_x$ and $\weights_t$. These components are illustrated on
\cref{fig:drag-drift-interp-2d}. Most importantly the drift component $\xi_x$ is approximately orthogonal to $\theta$ and its contribution to
$\norms{\theta_{t+1}}$ is quadratic in $\lr$, while the drag component, has a linear contribution. Let us compute the change in norms given $x$: 
\begin{align}
\norms{\weights_{t+1}} - \norms{\weights_t} &=  (1 - \lr \gamma_x)^2 \norms{\weights_{t}} + \lr^2 \norms{\xi_x} - \norms{\weights_t}  \nonumber \\ 
&= -\lr\gamma_x(2 -\lr \gamma_x) \norms{\weights_t} + \lr^2 \norms{\xi_x}, &
\label{eq:delta-dist}
\end{align}
therefore the change in norm becomes zero when
\begin{equation}
\lr = \lr_x := \frac{2\gamma_x \norms{\weights_t}}{\norms{\xi_x} + \gamma_x^2 \norms{\weights_t}}.
\end{equation}
Now we can estimate the \emph{expected} change in norm for a fixed $\theta_t$ using \cref{eq:delta-dist}: 
\begin{align}
& \Exoversamples{\norms{\weights_{t+1}} - \norms{\weights_t} | \weights_t} = -2\lr\norms{\weights_t} \Exoversamples{\gamma_x|\weights_t} \nonumber \\ 
& \qquad + \lr^2 \Exoversamples{\norms{\xi_x} + \gamma_x^2\norms{\weights_t} | \weights_t}.
\end{align}

The expected change in norm becomes zero when
\begin{equation}
\label{eq:lr-for-theta}
    \lr = 2\norms{\weights_t}\frac{\Exoversamples{\gamma_x|\weights_t}}{\Exoversamples{\norms{\xi_x}} + \norms{\weights_t} \Exoversamples{\gamma^2_x | \weights_t}},
\end{equation}
or, equivalently, for fixed $\lambda$ the norm has zero expected change
\begin{align}
\label{eq:norm-given-fixed-lr}
    \norms{\weights_t}&  = \frac{\lambda \Exoversamples{\norms{\xi_x}|\weights_t}}{ 2 \Exoversamples{\gamma_x | \weights_t} - \lr\Exoversamples{ \gamma_x^2 | \weights_t}} \nonumber \\
    & = \frac{\lr \Exoversamples{\norm{\xi_x}^2| \weights_t} }{2\Exoversamples{\gamma_x| \weights_t}} + O(\lambda^2) .
\end{align}
For the special case when $\Exoversamples{A^T_xA_x} = I$, we have $\Exoversamples{\gamma_x|\theta_t} = I$,
and since $A_x$ is independent of $\theta_t$ we can assume that $\weights_{t, A_x}^\perp$ is orthogonal to $b_x$
and its magnitude is independent of $\theta_t$, therefore $\Exoversamples{\norms{\xi}|\theta} \approx  \Exoversamples{\norms{\xi}}$ and thus 
\begin{equation}
\label{eq:weights-given-lr}
\norm{\weights} =\sqrt{\Exoversamples{\norms{\xi}}} \sqrt{\lr/2} = \norm{c} \sqrt{\lr/2}.
\end{equation}
\end{proof}

For a fixed learning rate we can expect the distance to the global minimum to stabilize at a value proportional 
to $\sqrt{\lambda}$ and described by \cref{eq:weights-given-lr}.

This model turns out of to be very similar to additive batch-noise model of~\citet{Wu2019-bg} and~\citet{Schaul_undated-uj}, however instead of introducing it, we arrive to it from a empirical assumption that each batch optimizes its own loss function. Despite its simplicity it appears to capture well many of the phenomena of large deep neural networks training. 

An alternative derivation of the result above is based on Fokker-Planck equation for describing the evolution of the state density~\citep{Sato2014,Li2017,Chaudhari2018} and showing that the stationary solution of this equation is a Gaussian distribution (for which the samples are concentrated on a spherical shell).
While in general this scaling factor is opaque, there are several special cases allowing for further simplification of the equation.



\vspace{-.2cm}
\paragraph{Weight averaging: effect on the stationary distribution.}
\begin{table*}[t]
    \centering
    \begin{tabular}{l|c}
    \toprule
       Averaging method             & Equivalent learning rate for last $k$ steps \\
       \midrule
       \makecell[l]{Stochastic Weight Averaging \citep{Izmailov2018-wv}, \\
       window size $k$, cycle $c=1$} 
        &  $\lambda(i) \frac{k+1 - i}{k}$ 
        \\        
        \midrule
        Two point average $k$ step apart   & ${\lambda(i)}/2$  \\         
        \midrule
        Exponential Moving Average (decay $\delta$) &  $ \lr(i)(1 - (1 - \delta)^{k-i})$, where ($k \gg 1/\delta)$\\ 
\bottomrule
    \end{tabular}
    \caption{Equivalent learning rate schedules for different averaging methods. Here $i$ denotes the current step of SGD trajectory and $\lambda(i)$ is a learning rate pre-aggregation.  Even though EMA is computed over the entire trajectory, the contributions of points beyond $1/\delta$ decay exponentially, thus the equivalent schedule only need to be applied to the last $k \gg \frac{1}{\delta}$. }
    \label{tab:stochastic-methods}
\end{table*}

As we saw earlier, for a fixed learning rate the solution trajectory stabilizes an ellipsoid whose size is proportional to $\sqrt{\lambda}$, while  traversing it indefinitely. If we let it continue sufficiently long we end up with two solutions $\weights_1$ and $\weights_2$ that are samples on the same sphere. Therefore the average $\weights = (\weights_1 + \weights_2)/ 2$ has norm $\norm{\weights} = \norm{\weights} / \sqrt{2}$. Thus averaging two solutions with a higher learning rate $\lambda$ gives us practically identical solution as if following the trajectory from $\weights_1$ to $\weights_2$, but with the learning rate $\lr/2$. Geometrically, the interpolation between two points on a random trajectory results in a point inside the sphere (\cref{fig:drag-drift-interp-3d}), which translates to a lower loss. Remarkably similar effects are also observed for \imagenet{} as we show in~\cref{fig:momentary-vs-others}.
In~\cref{app:weight-averaging} we generalise this result to an arbitrary averaging kernel.

We note that previous interpretations proposed in~\citet{Kingma2014-vr} suggest
that weight averaging works by ``smoothing'' the model. Instead, we show that there is a natural geometric
interpretation that arises from the basic properties of stochastic gradient descent.  Similar intuition can also be applied to stochastic weight averaging (SWA; \citealp{Izmailov2018-wv}) and exponential moving averaging (EMA; \citealp{Kingma2014-vr}).

\paragraph{Weight averaging: equivalence to the reduced learning rate along the trajectory.}
As we saw above, aggregation is equivalent to a specific learning schedule in the  stationary regime. Here we consider another angle: we show that a similar result also holds when we look at window sizes, where we can assume that gradient on a fixed batch doesn't change much along the trajectory.  We formalize this result in the lemma below. 
\begin{lemma}{(for proof see ~\cref{app:weight-averaging})}
\label{lem:matching-gradients}
If gradient of the loss with respect to a fixed batch is approximately constant, then for SWA with cycle of length $c=1$,
two-point averaging and EMA, the equivalent learning schedule is described in~\cref{tab:stochastic-methods}.
\end{lemma}

Note that under our simplifying assumption of matching gradients, we have shown a much stronger result: not only losses match, but the actual trajectories match. In the actual training, as can be seen in~\cref{fig:divergence-of-gradients} the gradients, while staying fairly aligned, do exhibit some amount of divergence. Further the divergence in $\weights$ space is also significant (though much less then if batches were sampled  independently). Despite that, we can see from~\cref{fig:momentary-vs-others} that \emph{loss} exhibit remarkable match between different averaging schemas and learning rate schedules. We conjecture that the requirements of this lemma can be relaxed in favor of loss match, where instead of matching the full trajectory, the different methods reach different points on the loss surface.

To recap, in our model we have shown that \emph{both} stationary regime and early in the trajectory the weight aggregation has equivalent learning schedules. It remains a subject of future work to bridge the theory to include the \emph{entire} trajectory, as appears to be the case for real deep neural networks.

\begin{figure*}[ht]
     \centering
    \includegraphics[width=0.32\textwidth]{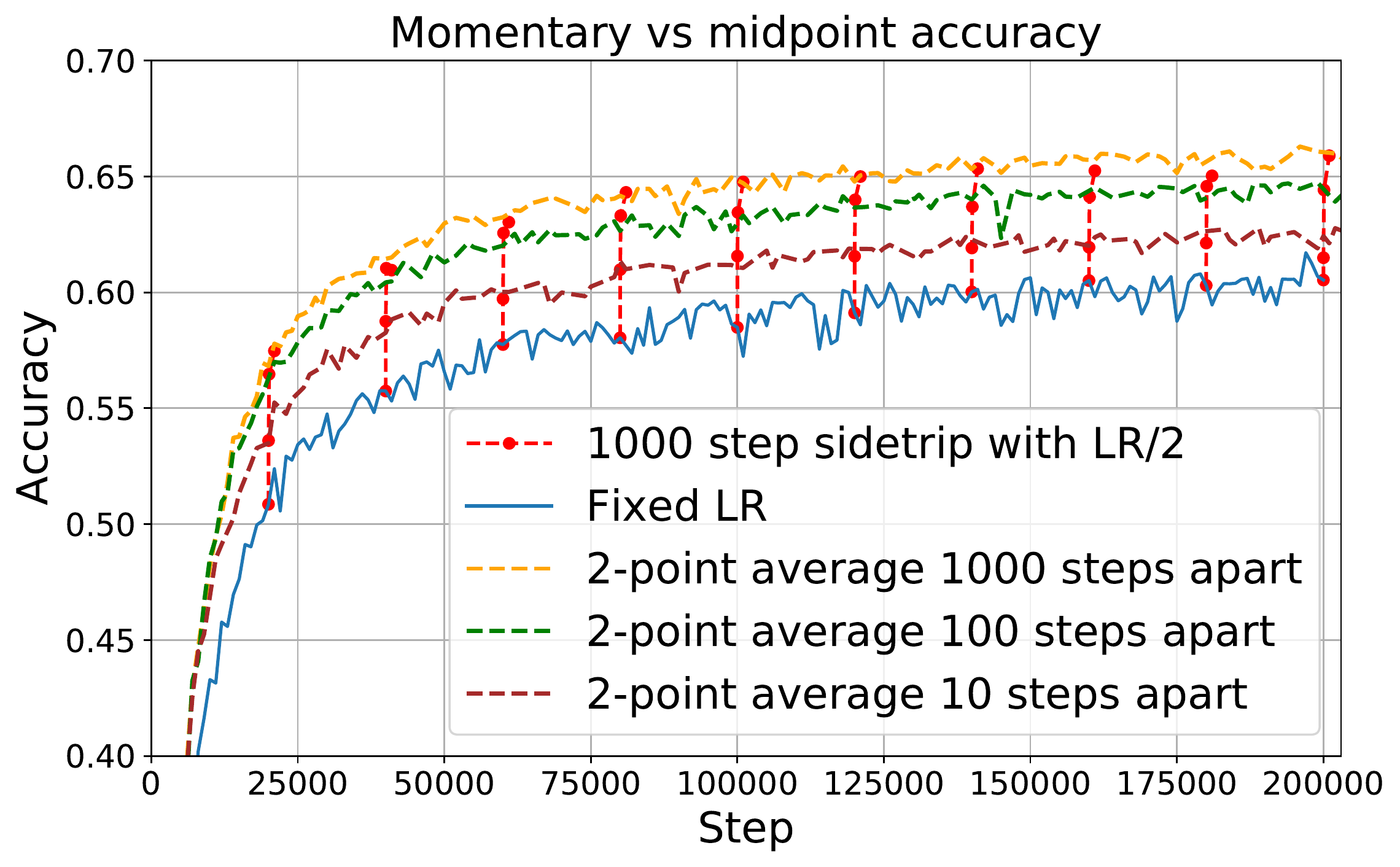}
    \includegraphics[width=0.32\textwidth]{figures/momentary_vs_average_0.1.pdf}
    \includegraphics[width=0.32\textwidth]{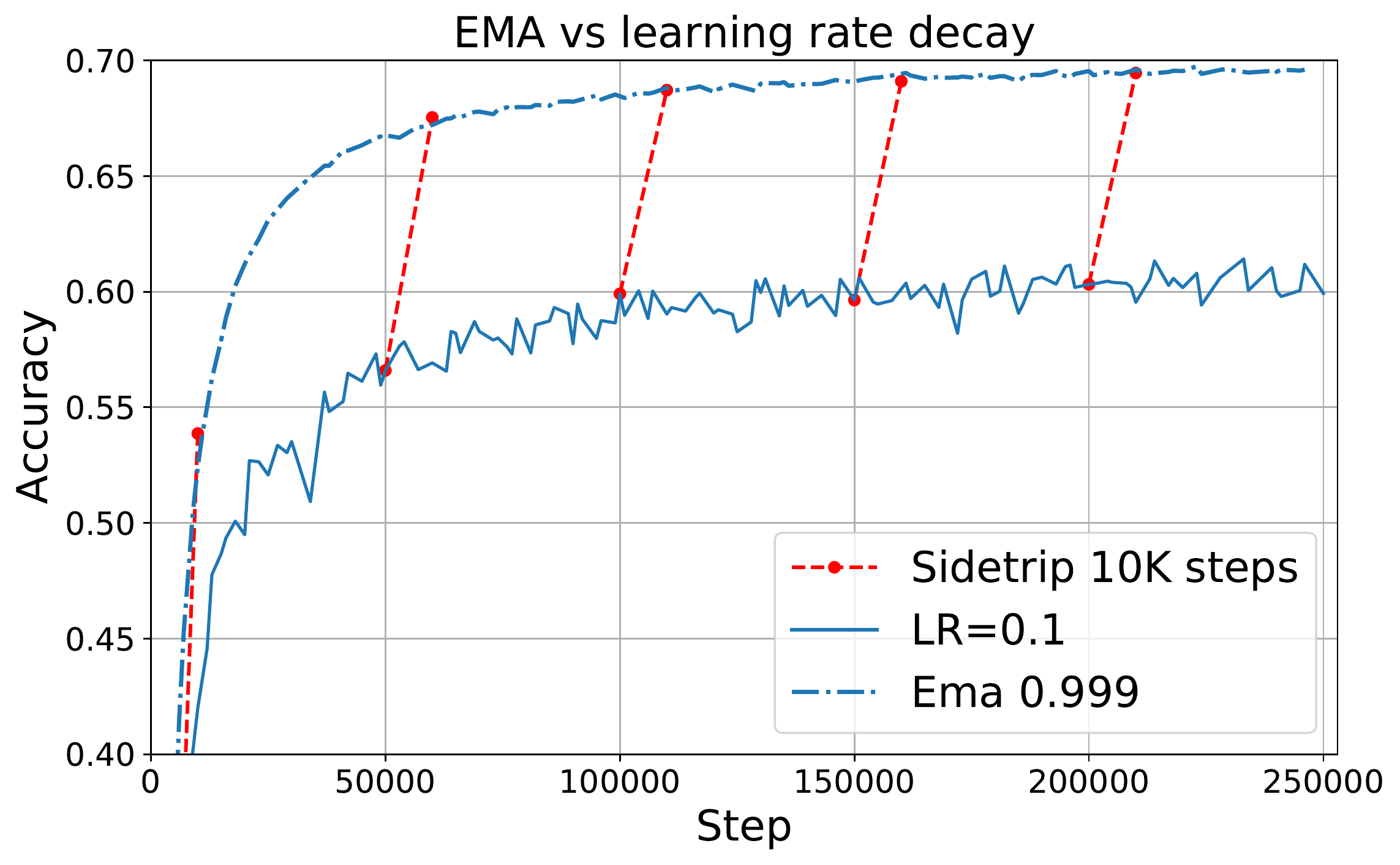}
     \caption{Comparing learning rate schedule with aggregation. The dotted vertical red lines show the set of  independent trajectories ``side-trips'' with appropriate learning rate schedule that start at corresponding point in the main trajectory. For midpoint and average, the red circles at midpoint show the accuracy at 1, 10, 100 and 1000 steps. For EMA the side-trips are 3000 steps. The alternating dash/dot show the running averages of the main trajectory (solid line, bottom).}
     \label{fig:momentary-vs-others}
 \end{figure*}
 
 \definecolor{mplblue}{rgb}{0.12109375, 0.46484375, 0.703125}
 \definecolor{mplorange}{rgb}{1.00, 0.50, 0.05}
 \definecolor{mplgreen}{rgb}{0.17, 0.63, 0.17,}
 \begin{figure*}[ht]
     \centering         
    \includegraphics[width=0.95\textwidth]{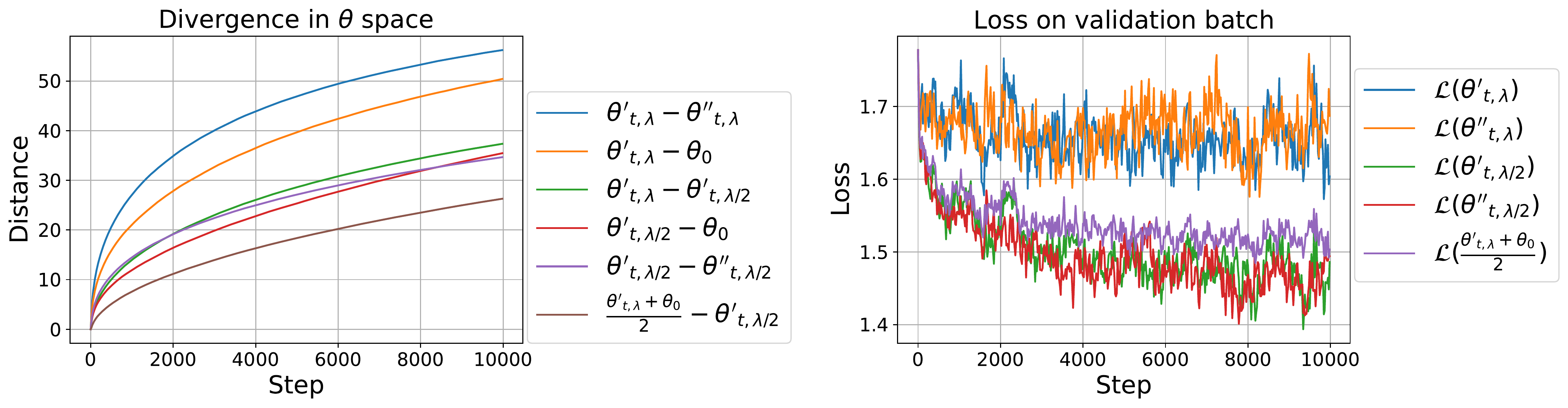}  
     \caption{Evolution of $L_2$ distances divergence  in weight space, the starting point $\weights_0$ was selected at 75K, with the starting learning rate at $\lambda=0.05$ and $\lambda/2$. $\theta'$ and $\theta''$ correspond to two different trajectories using differently sampled mini-batches. Note how the loss of the  average point of the  trajectory with learning rate $\lambda$ tracks the trajectory with $\lambda/2$.}
     \label{fig:distance-in-theta-space}
 \end{figure*}
\begin{figure*}[t]
     \centering         
    \includegraphics[width=0.99\textwidth]{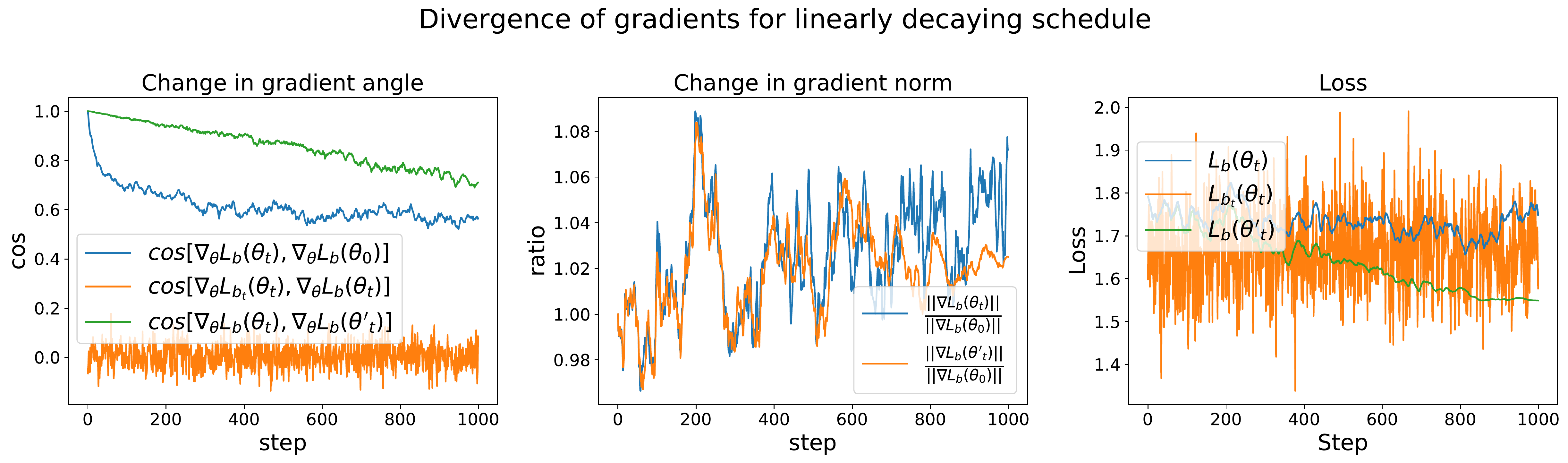}  
     \caption{Evolution of $L_2$ gradient divergence, the starting point $\theta_0$ was selected at 75K
     for different LR schedules. We compare fixed batch $b$ and training batch $b_t$.  For fixed batch the direction of the gradient changes very slowly over 1000 steps ({\bf\color{mplblue}blue curve}), while for different batches ({\bf\color{mplorange} orange curve}) all nearly orthogonal even at the \emph{same step}. The top most ({\bf \color{mplgreen} green curve}) on the left chart is the $\cos$ of the angle between two gradients at trajectory $\theta_t$ - fixed learning rate of $0.05$, and $\theta'_t$ - the linearly decaying learning rate $\lambda=0.05\frac{1000-i}{1000}$. The middle graph shows the change in gradient norm between different schedules. The rightmost graph, ({\bf\color{mplorange}orange curve}) shows the losses for training batch, while the {\bf\color{mplgreen}green curve}  and {\bf\color{mplblue}blue curve} show the loss of a fixed batch $b$ for constant and linearly decaying schedule respectively. Note that while green and blue curves maintain good gradient alignment, despite green curve exhibits significant loss change.  Best viewed in color.
}
     \label{fig:divergence-of-gradients}
 \end{figure*}

\vspace{-.2cm}
\section{Experiments}
\label{sec:experiments}
In this section we describe our additional experiments. Following our setup from \cref{sec:empirical}, we experiment with \imagenet, but also include results from Cifar10 and Cifar100~\citep{cifar10}.  For all datasets we use identical setup with ResNet-34, and use SGD with momentum 0.9. 

\paragraph{Different aggregation and learning rate schedule on \imagenet~and CIFAR datasets }
In this experiment we show that learning rate schedules described in \cref{tab:stochastic-methods} match stochastic averaging on several large datasets. 
On \cref{fig:momentary-vs-others} we show our results for the \imagenet. On \cref{fig:momentary-vs-others} we show validation accuracy, howeve similar results also hold for training splits as well for the actual losses, as we show in supplementary materials~\cref{fig:train-loss-const-cosine-lr}. Additionally  we include the results on Cifar10 and Cifar100 in the supplementary materials~\cref{fig:cifar10-100}.   

On \cref{fig:distance-in-theta-space} we show  how the solution for two-point averaging diverges from the equivalent learning rate schedule. As can be seen, the training loss matches our estimate fairly close, while arriving at two distinct solutions. Even though the average point is significantly apart, in relative terms the two solutions are much closer than two independently trained solution with identical learning rate. 

\paragraph{Gradient alignment and divergence of trajectories }
As we recall from~\cref{lem:matching-gradients}, we relied on the gradients
for a fixed batch essentially unchanging as we move along the trajectory. On \cref{fig:divergence-of-gradients} we show that gradients on the fixed batch $b$ indeed change very little. We consider two \emph{side-trips}, using standard sequence of mini-batches, but we measure the gradient on a \emph{fixed} mini-batch and compare its direction with the gradient at the initial step. As can be seen, the gradient norm changed less than 5\%, while $\mathop{cos}$ stayed generally above 0.5. For high dimensional space this means that the gradient on a fixed batch stays within a very narrow cone \emph{throughout} the trajectory. Remarkably, the angle between two parallel points of two trajectories (green curve) is even closer aligned than the gradient between the end-point and the start point. 

For reference we also include the angle between two \emph{independent} batches -- as one can see those gradients are typically almost perfectly orthogonal. All this suggests that we can generally expect the requirements of lemma \cref{fig:divergence-of-gradients} to hold even on large datasets.

\paragraph{Basins of attraction}
\begin{figure}[t]
    \centering
    \includegraphics[width=0.32\textwidth]{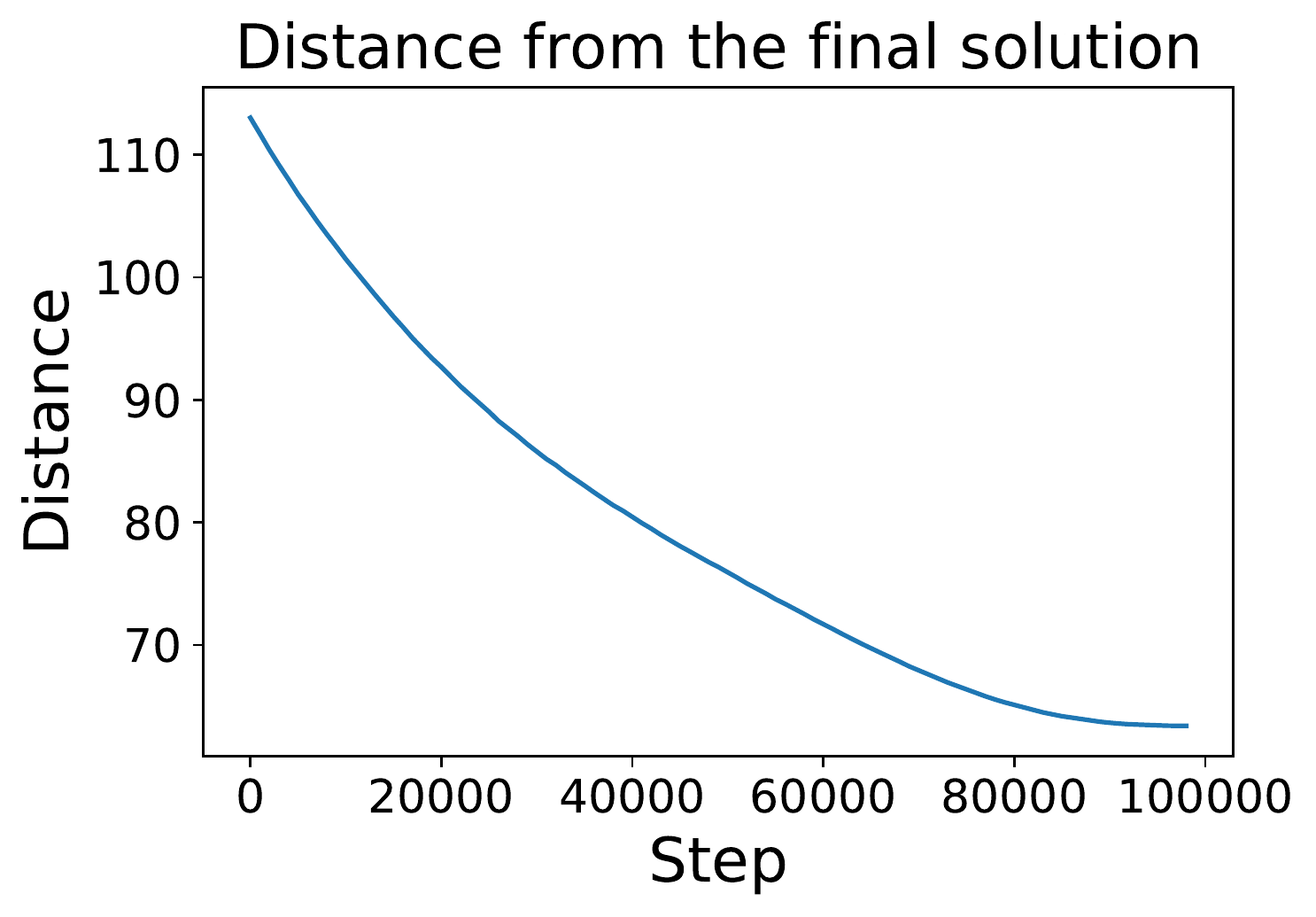}
    \label{fig:basins-of-attraction}
    \caption{Distance between intermediate weights $\weights_t$ and an independently
    trained alternative solution $\theta'$}
\end{figure}   

It has been known that multiple trajectories sharing initial segment end up in the attraction basin as measured by the absence of the loss barrier \citep{linear-mode-connectivity}. Here we show that in addition to being part of the same basin, the final point of one trajectory and an independent trajectory re-trained using different sequence of mini-batches monotonically decreases until saturates at some constant level. On \cref{fig:basins-of-attraction} we show how this distance changes for one such trajectory. This provides further evidence that independent trajectories land on a sphere around some fixed minimum. 
\subsection{Synthetic model experiments}

\paragraph{Fast and slow convergence}


We have previously shown that the effect of weight averaging can be simulated approximately by using a properly chosen learning rate schedule.
However, in practice, it is often advantageous to use fixed learning rate throughout training, especially if the new data is constantly arriving and the model needs to be trained continuously.
In this case, the weight averaging can improve model performance, while still allowing us to train the model with a large learning rate guaranteeing fast convergence.

The beneficial effect of weight averaging is particularly pronounced in situations where training dynamics is characterized by a wide spectrum of time scales, which is typical in practical scenarios \citep{Ghorbani2019}.
This rich spectrum of timescales 
manifests itself as the existence of elongated ``trenches'' in the loss landscape and long tails characteristic for the evolution of the loss during training.

Since weight averaging is generally sensitive to the time scale of the underlying weight evolution, one can expect it to have different effects on ``slow'' and ``fast'' degrees of freedom.
As discussed in Appendix~\ref{app:weight-averaging} in more detail, weight averaging will have little effect on the slow dynamics, but is expected to reduce the size of the stationary distribution for fast degrees of freedom and reduce the corresponding contribution to the loss as if the learning rate was in fact smaller.
Notice that using smaller learning rate without weight averaging would have a disadvantage of slowing down the convergence in slow coordinates and thus using large learning rate with weight averaging is expected to help us improve model performance without sacrificing the convergence speed.

We illustrate this intuition by solving~\eqref{eq:weight-update} for $b_x \sim \normalstd$ and a diagonal $\langle A^T_x A_x \rangle$ with values $1$ and $0.015$ for the fast and slow degrees of freedom correspondingly.
The evolution of the loss function in this system is shown in \cref{fig:multiscale} for two different learning rates: (a) $\lambda_0=5\cdot 10^{-2}$ and (b) $\lambda_1= 2\cdot 10^{-2}$.
Experiments with $\lambda=\lambda_0$ are conducted both with and without the exponential moving weight averaging (with a decay rate of $450$ steps).
\cref{fig:multiscale} illustrates that there are indeed two time scales in $L(t)$, fast and slow, and that weight averaging can dramatically reduce the stationary loss $L$ without sacrificing the speed of convergence, unlike when training with the smaller learning rate $\lambda=\lambda_1$.
Also, as shown in \cref{fig:multiscale-app}b, in the appendix, weight averaging has a much larger effect on the fast degrees of freedom, significantly reducing corresponding contributions to the overall loss, while having only a mild effect on slow coordinates.

\begin{figure}[t]
    \centering
    \includegraphics[width=0.4\textwidth]{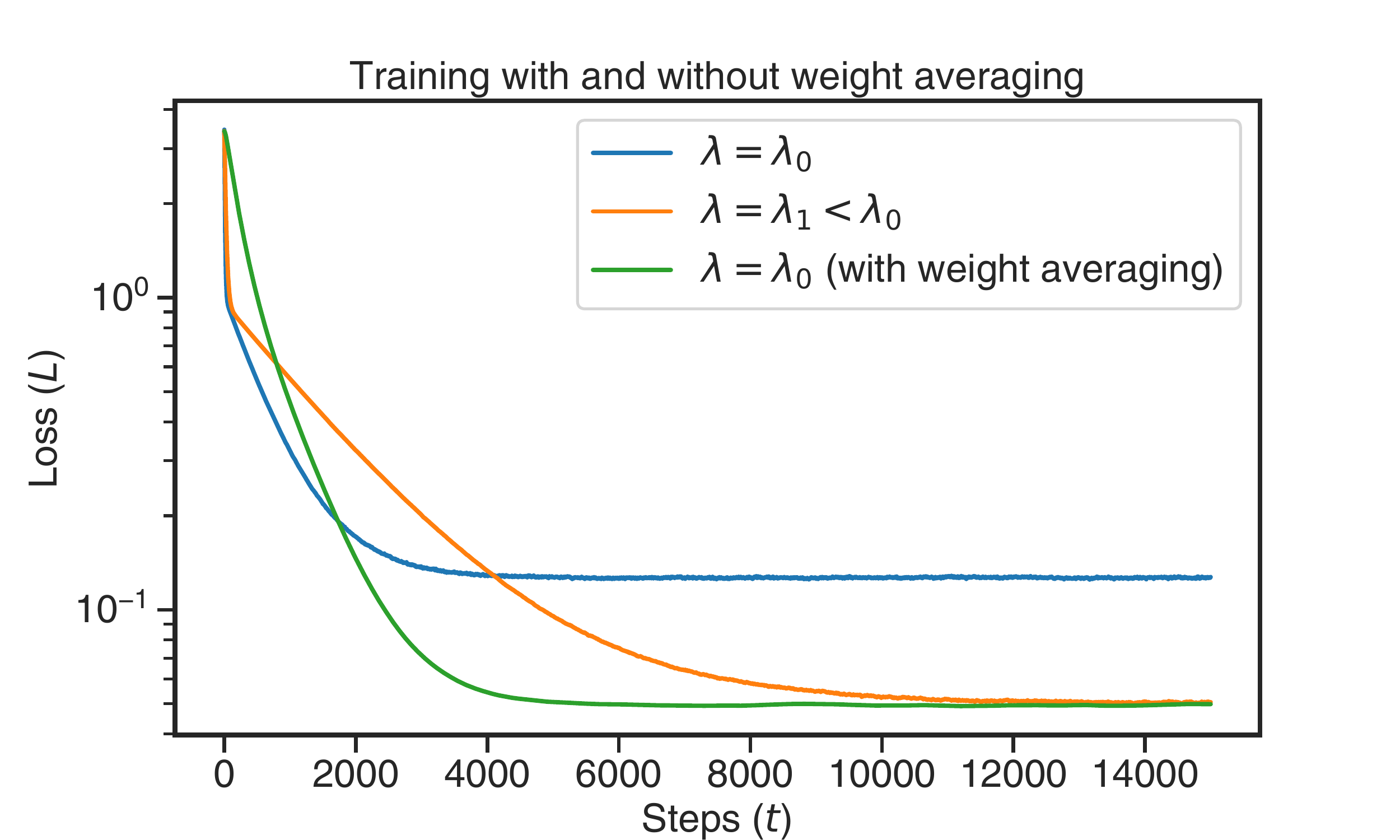}
    \label{fig:multiscale-a}
    \caption{
        Loss dynamics in a system governed by with ``fast'' ($\Omega_{ii}=1$) and ``slow'' ($\Omega_{ii}=0.015$) degrees of freedom. Larger ($\lambda_0=5\cdot 10^{-2}$) and smaller ($\lambda_1 = 2\cdot 10^{-2}$) learning rates v.s. EMA; 
    }
    \label{fig:multiscale}
\end{figure}

\section{Open question and conclusions}
\label{sec:conclusions}
We explored some remarkable properties of the learning rate in 
stochastic gradient descent, and demonstrated the novel connection between 
iterate averaging and learning rate schedules. We showed that this connection can be observed both in simple theoretical models
and in large-scale training on multiple datasets.
We hope that this work paves the way to further understanding on the role of the
learning rate in training. One direction that we see is developing more general models that would allow 
to study other phenomena commonly observed during training, such as overfitting and trajectory-dependent learning rate.

\clearpage
\bibliography{9_references}
\bibliographystyle{iclr2022_conference}

\appendix

\onecolumn
\begin{center}
{\Large \bf Supplementary materials for \\``Training trajectories, aggregation and the curious role of the learning rate''.}
\end{center}
\vspace{.5cm}
\hrule
\vspace{.5cm}
\section{Weight Averaging}
    \label{app:weight-averaging}

    In this section, we consider the evolution of $\weights$ governed by
    \begin{equation}
        \weights_{t+1} = \weights_t - \lr \lmat \, \weights_t - \lr \tilde{b}_{\batch},
        \label{eq:weights-simplified}
    \end{equation}
    a simplified version of the finite-batch version of \eqref{eq:weight-update}:
    \begin{equation*}
        \weights_{t+1} = \weights_t - \frac{\lr}{|\batch|} \sum_{x \in \batch} A^T_x A_x \weights_{t} - \frac{\lr}{|\batch|} \sum_{x\in \batch} b_x,
    \end{equation*}
    where $\batch$ is a batch of samples, $\lmat = \langle A^T_x A_x \rangle$ is a matrix and $\tilde{b}_{\batch}\sim \normal{\covariance}$ is a multivariate normal random variable with covariance $\covariance=\coord\coord^\top$.

    Here, we first show how \eqref{eq:weights-simplified} can be reduced even further by using a coordinate transformation that whitens the batch noise.
    We then consider a moving average of the training trajectory, derive a simple equation governing it and characterize the resulting converged steady-state distribution.

\subsection{Simplifying gradient descent equation}

    Stochastic gradient descent trajectories following \eqref{eq:weights-simplified} depend on both $\lmat$ and $\covariance$ matrices.
    Performing a change of coordinates $\weights_t = \coord u_t$, we can remove the dependence on one of the matrices if $\coord$ is an invertible matrix such that $\coord\coord^\top=\covariance$.
    Indeed, noticing that $\tilde{b}_\batch$ can be represented as $-\coord \xi_\batch$ with $\xi_\batch \sim \normaln(0,1)$, we obtain:
    \begin{equation*}
        \coord u_{t+1} = 
        \coord u_t - \lr \lmat \coord u_t + \lr \coord \xi_\batch,
    \end{equation*}
    and finally get a simpler description of the stochastic gradient descent trajectory:
    \begin{equation}
        \label{eq:simp-trajectory}
        u_{t+1} = (\id - \lr \Omega) u_t + \lr \xi_t,
    \end{equation}
    where our new equation depends on only one matrix $\Omega := \coord^{-1} \lmat \coord$.

    The solution of \eqref{eq:simp-trajectory} can be obtained recursively and reads:
    \begin{gather}
        \label{eq:finu}
        u_{t+1} = (\id - \lr \Omega)^{t+1-\tau} u_\tau + \lr \sum_{t'=\tau}^t (\id - \lr \Omega)^{t-t'} \xi_{t'}.
    \end{gather}
    
\subsection{Moving average of the SGD trajectory}

    In Section~\ref{sec:model}, we presented a simple intuitive explanation of the fact that the two-point average evolves similarly to an ordinary SGD trajectory with an effective learning rate of $\hat{\lr} = \lr/2$.
    Here we prove a more general result for an arbitrary averaging procedure.
    
    Consider a training trajectory $u_t$ solving \eqref{eq:simp-trajectory} and let $\hat{u}_{t} \equiv \sum_{k=0}^{\infty} \mu_k u_{t-k}$ be some average of the training trajectory with the real averaging kernel $\{\mu_k \in \R | k \in 0, \dots,\infty\}$.
    It is easy to see that $\hat{u}_{t}$ satisfies the following equation:
    \begin{gather}
        \label{eq:av-trajectory}
        \hat{u}_{t+1} =  (\id - \lr \Omega)\hat{u}_{t} + \lr \sum_{k=0}^{\infty} \mu_k \xi_{t-k}.
    \end{gather}
    Introducing $\nu_t \equiv \sum_{k=0}^{\infty} \mu_k \xi_{t-k}$, we see that this random process is characterized by $\langle \nu_t \rangle = 0$ and an autocorrelation $C_{t,\delta} \equiv \langle \nu_t \nu_{t+\delta} \rangle = \sum_{k=0}^{\infty} \mu_k \mu_{k+\delta}$ independent of time $t$, where the averaging is performed over different realizations of the batch noise $\xi$.
    
    The solution of \eqref{eq:av-trajectory} can be obtained by analogy with how solution \ref{eq:finu} was obtained for \eqref{eq:simp-trajectory}:
    \begin{gather*}
    \hat{u}_t = \Gamma^t \hat{u}_0 + \lr \sum_{\tau=1}^{t} \Gamma^{t-\tau} \nu_\tau,
    \end{gather*}
    where we introduce $\Gamma \equiv \id - \lr \Omega$.
    Choosing $t=T$ for some sufficiently large $T$ and assuming that all eigenvalues $\kappa_k$ of $\Gamma$ are characterized by $0 < \kappa_k < 1$, the first term ends up being exponentially small $\sim \exp(-\kappa_{\rm max} T)$ and
    \begin{gather*}
    \hat{u}_T \approx \lr \sum_{\tau=0}^{T-1} \Gamma^{\tau} \hat{\nu}_{\tau},
    \end{gather*}
    where $\hat{\nu}_\tau \equiv \nu_{T-\tau}$.
    
    Now let us look at the distribution $p(\hat{u}_T)$ for a fixed time $T$ and different realizations of $\nu$.
    It is easy to see that $\langle \hat{u}_{T} \rangle = 0$ since $\langle \hat{\nu}_\tau \rangle=0$ for each $\tau$.
    Furthermore,
    \begin{gather*}
    \langle \hat{u}_{T} \hat{u}_{T} \rangle = \lr^2
    \left\langle
    \sum_{\tau,\tau'=0}^{T-1} \Gamma^{\tau+\tau'} \hat{\nu}_{\tau}
    \hat{\nu}_{\tau'} 
    \right\rangle =
    \lr^2
    \sum_{\tau,\tau'=0}^{T-1} \Gamma^{\tau+\tau'}
    \left\langle \hat{\nu}_{\tau} \hat{\nu}_{\tau'} \right\rangle.
    \end{gather*}
    Notice that here we average over different realizations of $\nu_t$ and assume homogeneity of $\nu_t$ in time since $C_{t,t'-t}$ is independent of $t$ and only depends on the difference $t'-t$.
    Assuming that $T$ is sufficiently large for both $\kappa_{\rm max} T \gg 1$ and $C_{t>T} \ll C_0$ to hold, we can approximate (changing the integration bounds where $\Gamma^{\tau+\tau'} C_\delta$ is vanishingly small):
    \begin{gather*}
    \langle \hat{u}_{T} \hat{u}_{T} \rangle \approx
    \lr^2 \left(
    \sum_{\tau=0}^{T-1} \Gamma^{2\tau} C_0 + 2 \sum_{\delta=1}^{T-1}
    \sum_{\tau=0}^{T-1} \Gamma^{2\tau+\delta} C_\delta
    \right).
    \end{gather*}
    Here the coefficient of $2$ emerges because $C_\delta=-C_\delta$ for $\delta > 0$.
    This can also be rewritten as (replacing $T$ with $T+1$):
    \begin{gather}
    \langle \hat{u}_{T+1} \hat{u}_{T+1} \rangle \approx
    \lr^2 X_T \left( C_0 +
    2 \sum_{\delta=1}^{T} C_\delta \Gamma^\delta \right),
    \label{eq:uu-a}
    \end{gather}
    where $X_T \equiv \sum_{\tau=0}^{T} \Gamma^{2\tau}$.
    
    Now let us look at the limit of $T\to \infty$, assuming for simplicity that $C_\delta$ is bounded as $\delta \to \infty$ and assuming that the eigenvalues of $\Gamma=\id-\lr \Omega$ are all smaller than $1$.
    The expression for $X_\infty$ can be simplified by relating it to Taylor series for $1/(1-x)$, specifically notice that:
    \begin{gather*}
      (\id-\Gamma^2) X_\infty =
      (\id-\Gamma^2) \sum_{\tau=0}^{\infty} \Gamma^{2\tau} =
      \sum_{\tau=0}^{\infty} \Gamma^{2\tau} - \sum_{\tau=1}^{\infty} \Gamma^{2\tau} = \id
    \end{gather*}
    and therefore $X_\infty = (\id-\Gamma^2)^{-1}$.
    We can then approximate \eqref{eq:uu-a} as follows:
    \begin{gather*}
    F \equiv \lim_{T\to \infty} \langle \hat{u}_{T+1} \hat{u}_{T+1} \rangle \approx
    \lr^2 (\id-\Gamma^2)^{-1} \left( C_0  +
    2 \sum_{\delta=1}^{\infty} C_\delta \Gamma^\delta \right),
    \end{gather*}
    or recalling that $\Gamma=\id - \lr\Omega$:
    \begin{gather}
    F \approx
    \lr (2 \Omega - \lr \Omega^2)^{-1} \left( C_0  +
    2 \sum_{\delta=1}^{\infty} C_\delta \Gamma^\delta \right).
    \label{eq:f-final}
    \end{gather}
    
    This final expression for $F$
    \internal{(verified numerically with a good agreement -- see {\bf \href{https://colab.corp.google.com/drive/1hGlbrbqptT9xlIKe0_8fZh_hOTLqnnsC}{the notebook}})}
    connects the covariance of the stationary distribution in the weight space to the autocorrelation of the averaging kernel $\mu_t$.
    One important conclusion is that since $\id-\lr \Omega$ is contracting (assuming $\Omega$ is positive-definite), the long tail of $C_\delta$ can be effectively cancelled by $(\id - \lr\Omega)^{\delta}$.
    In other words, if the characteristic averaging time exceeds $1/\kappa$, where $\kappa$ is the eigenvalue of $\id - \lr\Omega$, then the corresponding width of the stationary distribution will be inhibited by weight averaging.

\subsection{Two-point average}

    The final expression \ref{eq:f-final} can then be applied to an arbitrary averaging procedure.
    For example, the following result generalizes the geometric derivation for the two-point average to an arbitrary\footnote{$\lr$ still needs to be sufficiently small for system dynamics to not diverge} learning rate $\lr$ and the distance $\Delta$ between steps.

    Since we define the averaging procedure via $\mu_0=1/2$ and $\mu_\Delta=1/2$ (and $\mu_k=0$ otherwise), we obtain $C_0=1/2$ and $C_\Delta=1/4$.
    Substituting this expression into \eqref{eq:f-final}, we obtain:
    \begin{equation*}
        F_\Delta \equiv \lim_{T\to \infty} \langle \hat{u}_{T} \hat{u}_{T} \rangle \approx S_\lr \frac{\id + (\id - \lr \Omega)^{\Delta}}{2},
    \end{equation*}
    where $S_\lr \equiv \lr (2 \Omega - \lr \Omega^2)^{-1}$.
    This expression can also be interpreted as:
    \begin{equation}
        \label{eq:eff-lam}
        \hat{\lr} = \lr \frac{1 + (1 - \lr \kappa)^{\Delta}}{2},
    \end{equation}
    where $\hat{\lr}$ is the effective learning rate for a degree of freedom corresponding to an eigenvector of $\Omega$ with an eigenvalue of $\kappa$.
    This expression generalizes our previous observation that $\hat{\lr}\approx \lr$ for sufficiently small $\Delta \sim 1$ and $\hat{\lr} \approx \lr/2$ for $\Delta \gg 1$ when $(1 - \lr \kappa)^\Delta \approx 0$.

\subsection{Multi-point average}

    Now let us study weight averaging over $n$ points separated by $\Delta$ steps each, i.e., $\hat{u}_t = n^{-1}(u_t + u_{t-\Delta} + u_{t-2\Delta} + \dots + u_{t-(n-1)\Delta})$.
    Since this corresponds to $\mu_{k\Delta}=1/n$ for $k=0,\dots,n-1$, we obtain:
    \begin{gather*}
        C_{k\Delta} = \sum_{\tau=0}^{\infty} \mu_{k\Delta} \mu_{(k + \tau)\Delta} = \frac{n-k}{n^2}.
    \end{gather*}
    Substituting this expression in \eqref{eq:f-final}, we obtain:
    \begin{gather}
    C_0 + 2 \sum_{\delta=1}^{\infty} C_\delta \Gamma^\delta =
    \frac{1}{n} + \frac{2}{n} \sum_{k=1}^{n-1} \left( 1 - \frac{k}{n} \right) \Gamma^{k\Delta} =
    \frac{1}{n} + \frac{2}{n} \left(G_1 - \frac{G_2}{n}\right),
    \label{eq:exp-for-c}
    \end{gather}
    where
    \begin{gather*}
    G_1 \equiv \sum_{k=1}^{n-1} \Gamma^{k \Delta}, \qquad
    G_2 \equiv \sum_{k=1}^{n-1} k \Gamma^{k \Delta}.
    \end{gather*}
    It is not difficult then to verify that:
    \begin{gather*}
    G_1 = (\id - \Gamma^\Delta)^{-1} \Gamma^{\Delta} \left(\id - \Gamma^{(n-1)\Delta} \right), \\
    G_2 = (\id - \Gamma^\Delta)^{-2} \Gamma^{\Delta} \left(\id - n \Gamma^{(n-1)\Delta} + (n-1) \Gamma^{n\Delta} \right).
    \end{gather*}
    These expressions for $G_1$ and $G_2$ together with \eqref{eq:f-final}  and \eqref{eq:exp-for-c} fully define the covariance $F$ for arbitrary values of $\lr$, $\Delta$ and an arbitrary $\Omega$.
    
    As before, in the limit of $\Delta \to \infty$ when $\Gamma^\Delta \to 0$, we see that
    \begin{gather*}
        F \approx
        \frac{S_\lr}{n} \left( \id  + \frac{2(n-1)}{n} \Gamma^{\Delta} \right) + o(\Gamma^\Delta).
    \end{gather*}
    In other words, as one would expect, the effective learning rate is decreased by approximately a factor of $n$.

\subsection{Proof of Lemma~\ref{lem:matching-gradients}}

    \begin{proof}
    We have $\weights_{t_2,\lr(t)} = \weights_{t_1} + \sum_{i={t_1}}^{t_2} \lambda(i) \xi_{x_i}(\theta_i)$, 
    where $\xi_{x_i}$ is a stochastic gradient of batch $x_i$ with respect to $\theta_i$. 
    thus 
    $$\frac{\weights_{t_2}  + \weights_{t_1} }{2} = 
    \weights_{t_1} + \sum_{i={t_1}}^{t_2-1} \frac{\lambda(i)}{2} \xi_{x_i} \approx \weights_{t_2, \lr(t)/2},$$
    where we assumed that for fixed $x_i$, the gradient $\xi_{x_i}$ changes very slowly and thus  $\xi_{x_i,\lr(t)} \approx \xi_{x_i, \lr/2(t)}$.
    
    Similarly for stochastic weight averaging we have 
    \begin{align}
    \theta_{\text{swa}} &= \frac{1}{t_2-t_1}\sum_{t=t_1}^{t_2} \weights_t \\
    & = \frac{1}{t_2 - t_1}\sum_{t=t_1}^{t_2}\left[\weights_{t_1} + \sum_{i=t_1}^{t-1} \lambda(i) \xi_{x_i}\right] \\
    &= 
    \weights_{t_1} +  \sum_{t=t_1}^{t_2-1} \lambda(t) \frac{t_2 - t}{t_2 - t_1} \xi_{x_t} 
    \approx \weights_{t_2,\frac{t_2 - t}{t_2 - t_1}\lr(t)}
    \end{align}
     Finally, for exponential moving average we have:
     \begin{eqnarray}
     \weights_{\text{ema}, t_2} & = &    \delta\sum_{t=0}^{t_2} \weights_t (1 - \delta)^{t_2 - t } \\
     &\approx & =
      \delta\sum_{t=t_1}^{t_2} \weights_t (1 - \delta)^{t_2 - t }\\
      &= &\delta[\weights_{t_1} \sum_{t=t_1}^{t_2} (1-\delta)^{t_2-t} + \\
      && + \sum_{t=t_1}^{t_2-1} \lambda(t) \xi_{x_t} \sum_{j=t}^{t_2-1} (1-\delta)^{t_2 - j-1}] \\
    &\approx& \weights_{t_1} + \sum_{t=t_1}^{t_2-1}\lr(t)\left(1 - (1-\delta)^{t_2-t}\right)\xi_{t}
    \end{eqnarray}
    where the first  transition holds because we assume $t_2-t_1 \gg 1/\delta$, and thus contribution of terms
    up-to $t_1$ is negligible. 
    \end{proof}

\section{Additional Experiments}
    Comparison of a two-point average and an equivalent learning schedule is shown in \cref{fig:cifar10-100}.
    \begin{figure*}[t]
        \centering
        \begin{subfigure}{0.49\textwidth}
        \includegraphics[width=0.98\textwidth]{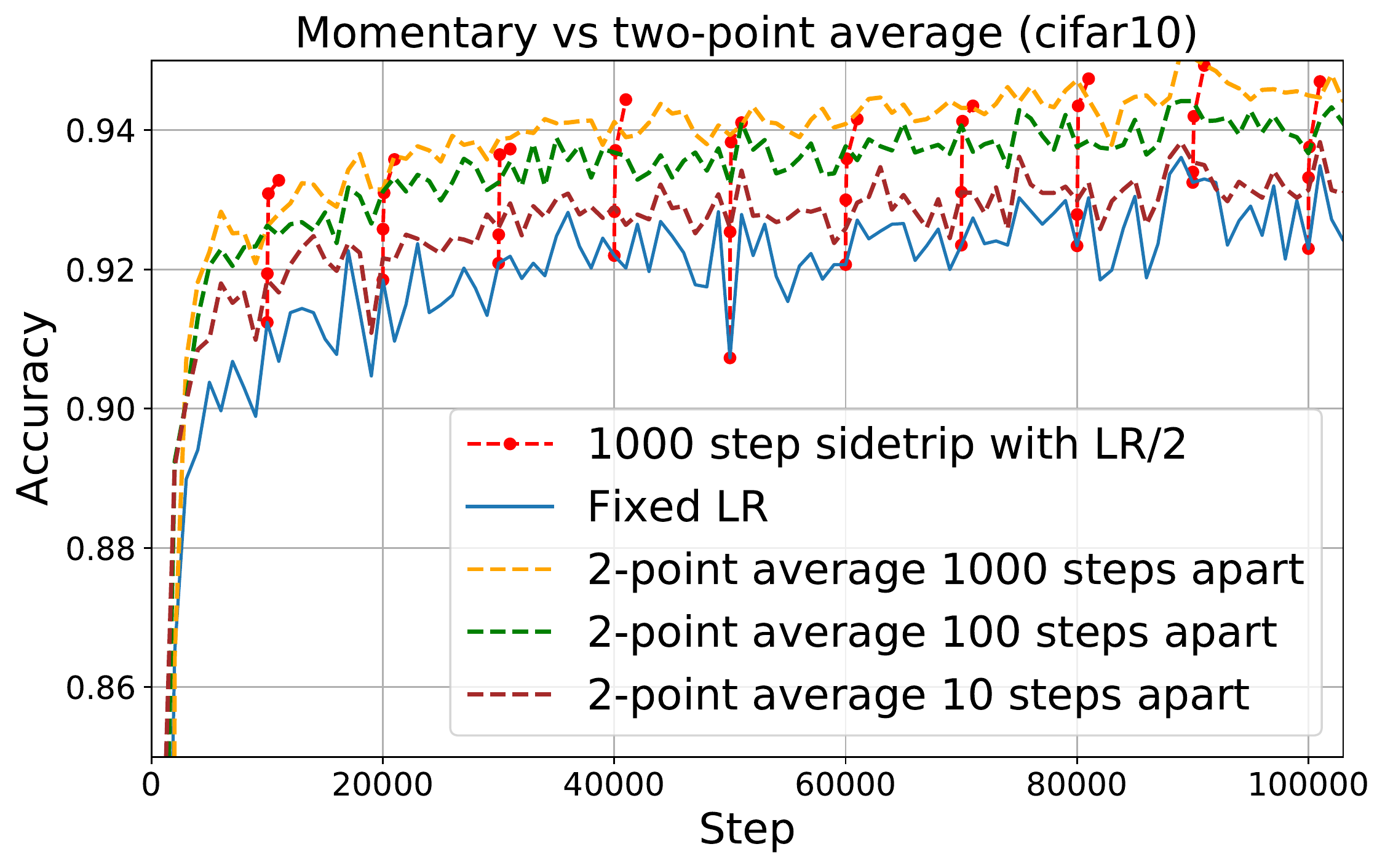}
        \caption{}
    
        \end{subfigure}
        \begin{subfigure}{0.49\textwidth}
        \includegraphics[width=0.98\textwidth]{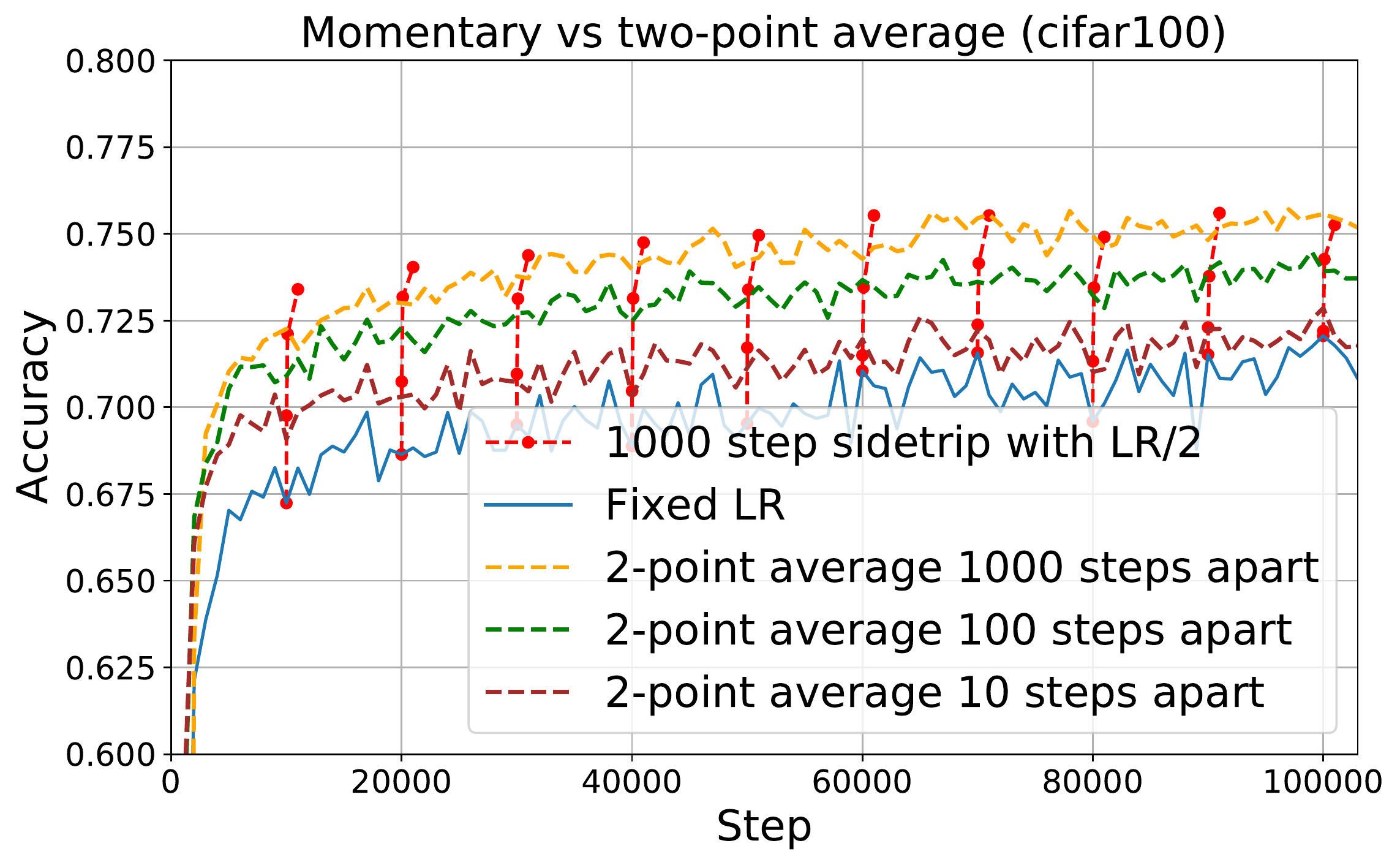}
        \caption{}
        \end{subfigure}
        \caption{Two-point averaging on validation data for CIFAR10 and CIFAR100}
        \label{fig:cifar10-100}
    \end{figure*}    
 \begin{figure*}[t]
    \centering
    \begin{subfigure}{0.4\textwidth}
    \includegraphics[width=0.98\textwidth]{figures/multiscale_cmp.pdf}
    \caption{}
    \label{fig:multiscale-a-app}
    \end{subfigure}
    \begin{subfigure}{0.4\textwidth}
    \includegraphics[width=0.98\textwidth]{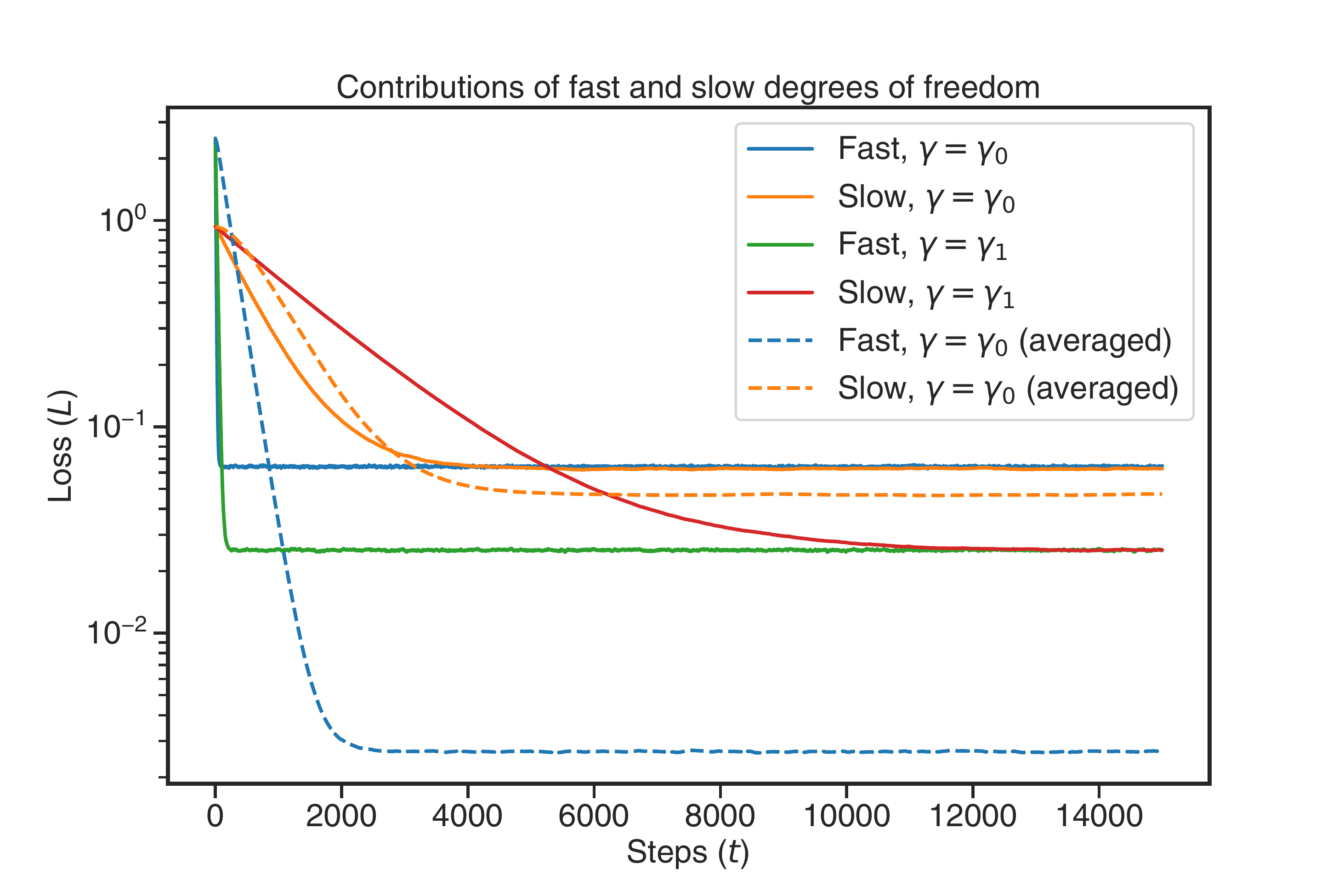}
    \caption{}
    \label{fig:multiscale-b-app}
    \end{subfigure}
    \caption{
        Loss dynamics in a system governed by with ``fast'' ($\Omega_{ii}=1$) and ``slow'' ($\Omega_{ii}=0.015$) degrees of freedom. (a) larger ($\lambda_0=5\cdot 10^{-2}$) and smaller ($\lambda_1 = 2\cdot 10^{-2}$) learning rates vs. Exponential Moving Average; (b) the effect of exponential moving averaging of the model weights for fast and slow degrees of freedom.
    }
    \label{fig:multiscale-app}
\end{figure*}

\subsection{Does single-step behavior findings generalize to different batch sizes, or at different trajectory}
In this section we show that our observations from the main experimental 
hold both at different stages of trajectory and for different batch sizes. On \cref{fig:single-step-batches} we show how the loss changes in a single step when performing this step at different location of SGD trajectory and using different batch size.  On \cref{fig:max-change} we
show the maximum loss that is achievable if we pick a step size that would pick ``optimal'' step. In other words the points correspond
to minima of each line of \cref{fig:single-step-batches}.

\begin{figure*}[t]
\centering
\label{fig:single-step-batches}
    \begin{subfigure}{0.4\textwidth}
    \includegraphics[width=0.98\textwidth]{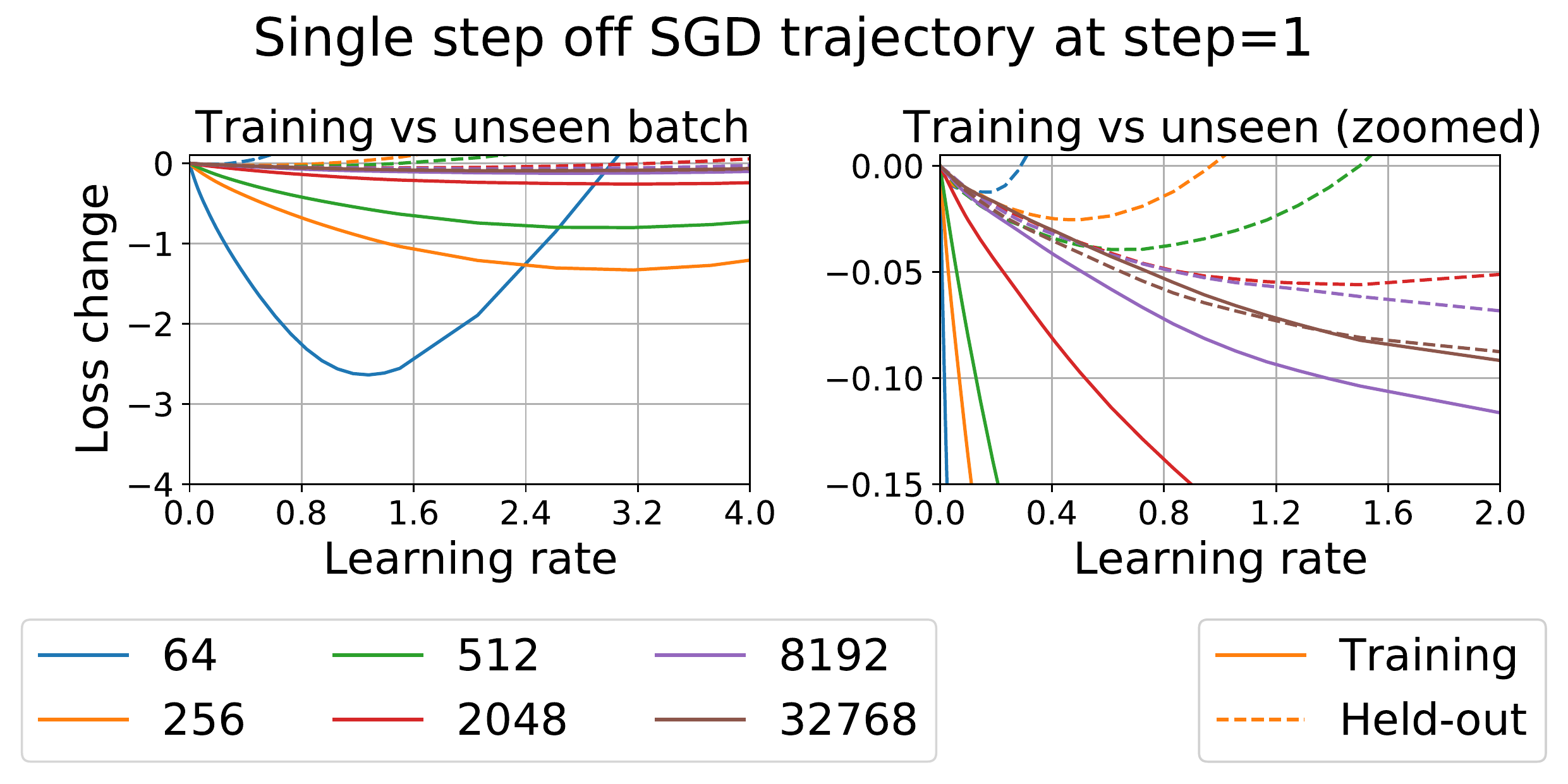}
    \caption{}
    \end{subfigure}
    \begin{subfigure}{0.4\textwidth}
    \includegraphics[width=0.98\textwidth]{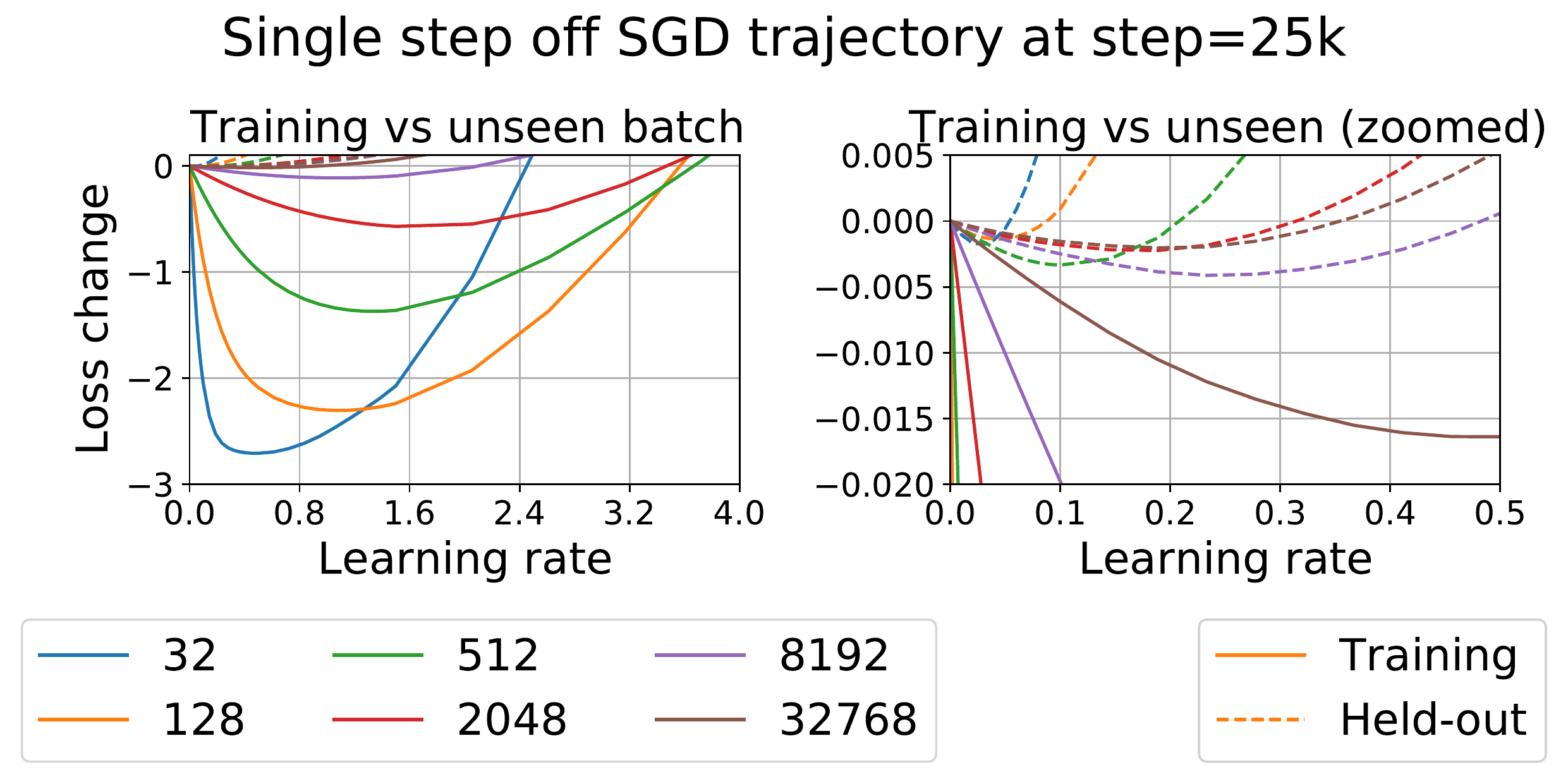}
    \caption{}
    \end{subfigure}    
    \begin{subfigure}{0.4\textwidth}
    \includegraphics[width=0.98\textwidth]{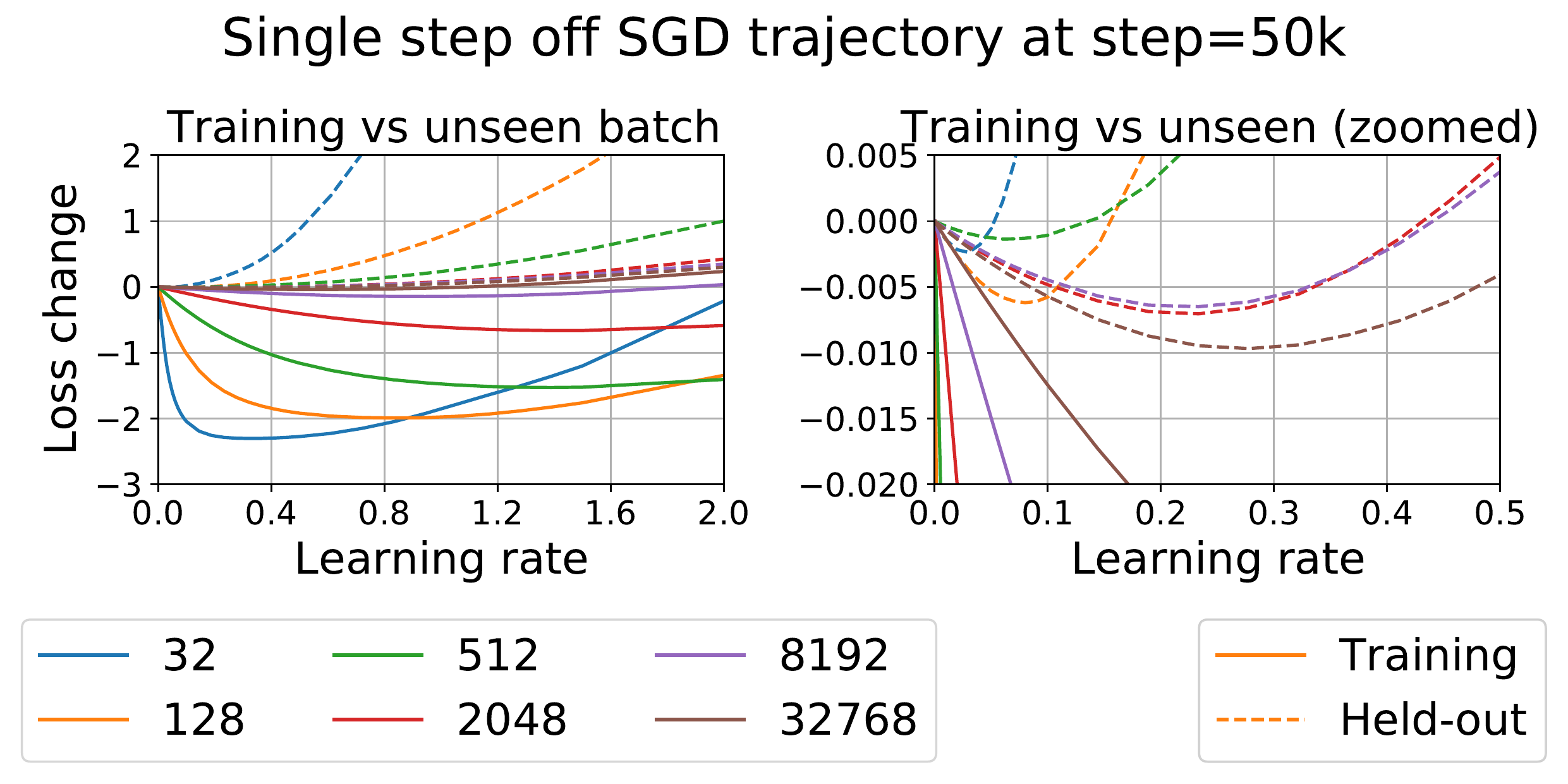}
    \caption{}
    \end{subfigure}
    \begin{subfigure}{0.4\textwidth}
    \includegraphics[width=0.98\textwidth]{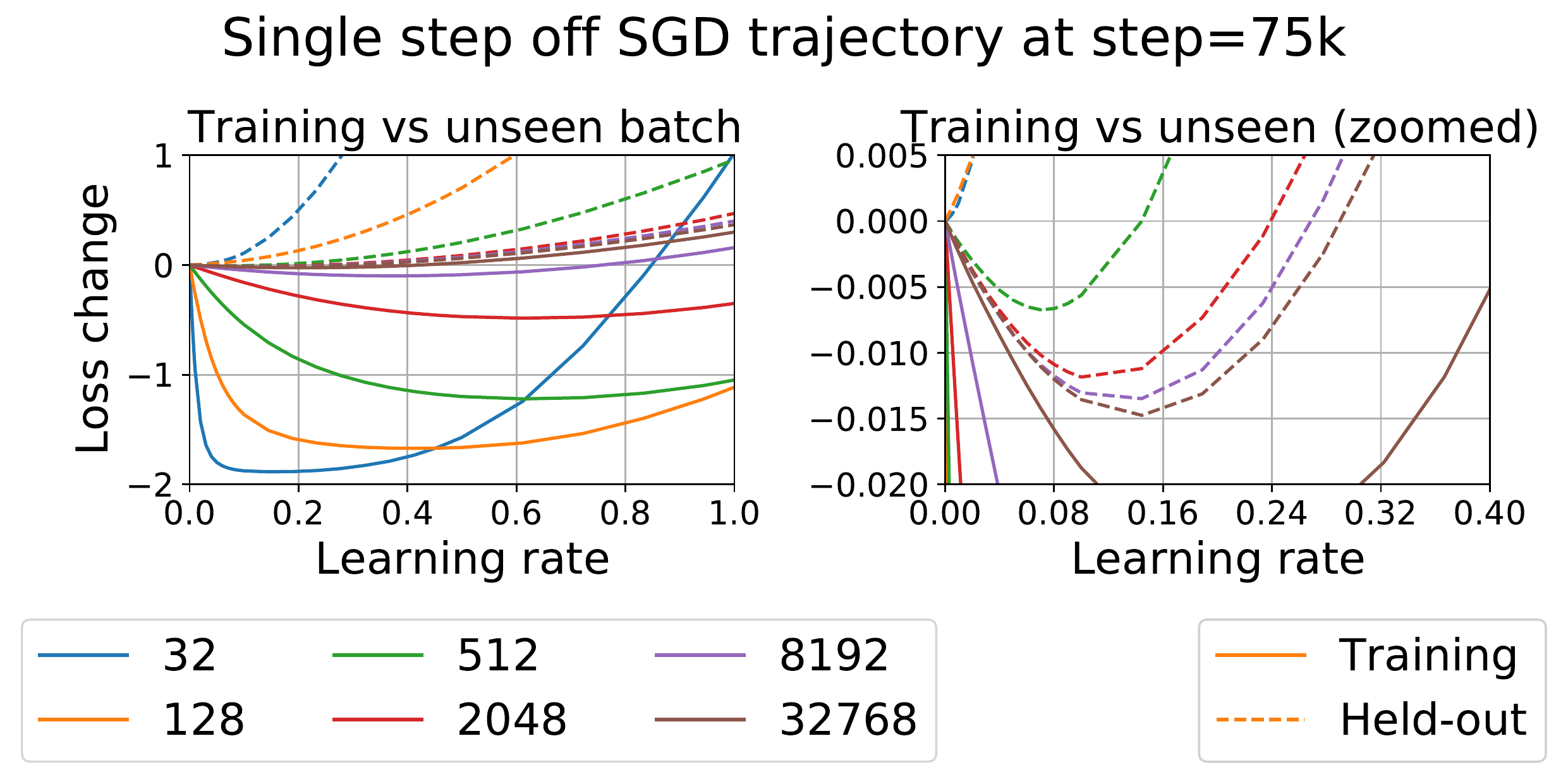}
    \caption{}
    \end{subfigure}
    \caption{Loss change as a function of step size starting off different points of SGD trajectory, and using different batches.}
\end{figure*}
 \begin{figure*}[t]
\centering
    \includegraphics[width=0.98\textwidth]{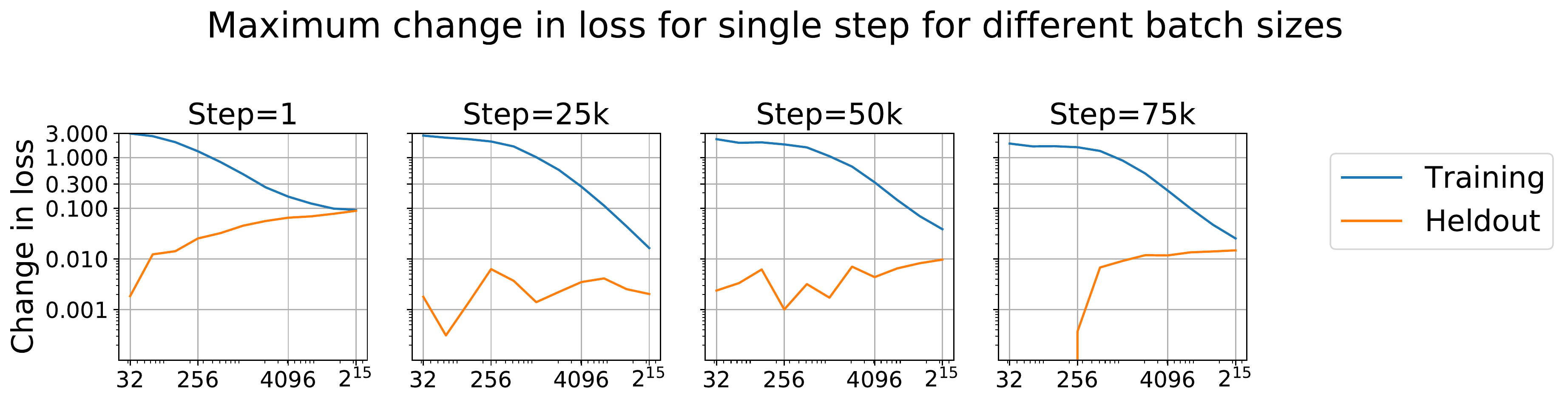}
\caption{Maximum loss drop  (log scale) for single step, when performed at different stages of SGD trajectory. As one can see even for batch size that are in the typical range (256-4096), training batch drops several orders of magnitude more than held out batch, however as the batch size reaches larger values such as 32768 the difference decreases. }    \label{fig:max-change}
\end{figure*}

\subsection{Additional experiments on Imagenet}
On \cref{fig:train-loss-const-cosine-lr} we show that sidetrip with appropriate learning rate schedule v.s. aggregation holds 
for actual loss and the accuracy, as well as for both training and test data. 
\begin{figure*}
\centering
\label{fig:train-loss-const-cosine-lr}
    \begin{subfigure}{0.98\textwidth}
    \includegraphics[width=0.98\textwidth]{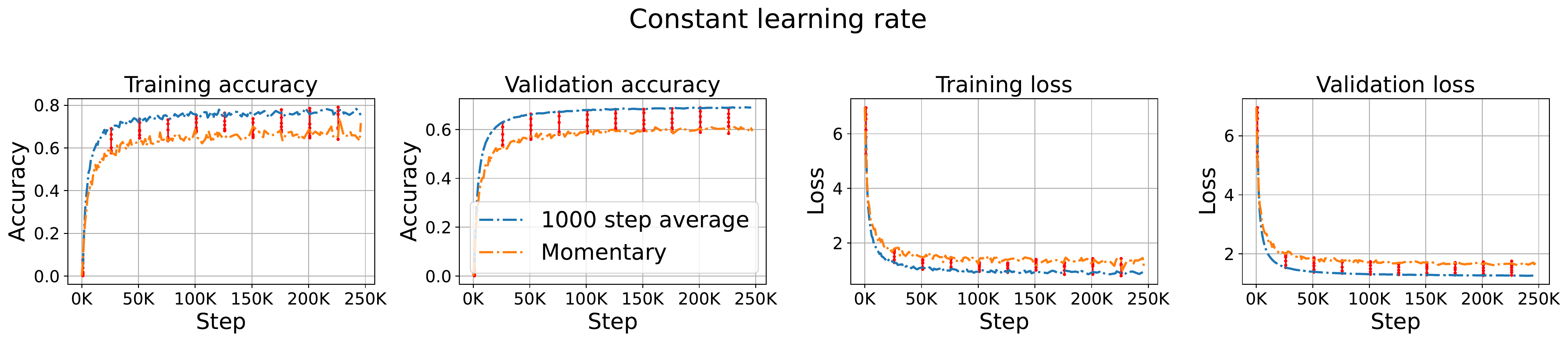}
    \end{subfigure}
\begin{subfigure}{0.98\textwidth}
    \includegraphics[width=0.98\textwidth]{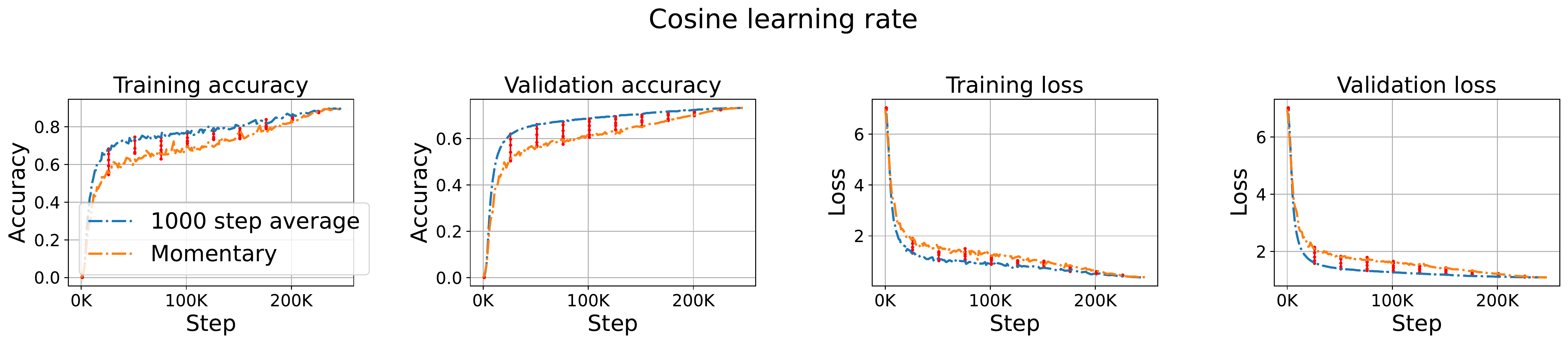}
    \end{subfigure}
\caption{Accuracy and loss for training and validation splits on Imagenet. The bottom row shows the trajectory for cosine learning rate schedule, which achieves state of the art accuracy for Resnet-34 architecture. Each graph shows momentary v.s. Stochastic Averaging (SWA) and side-trips (red) with appropriate learning rate schedule. }
\end{figure*}

\end{document}